\documentclass{article}

\usepackage[final,nonatbib]{neurips_2021}

\usepackage[utf8]{inputenc} %
\usepackage[T1]{fontenc}    %
\usepackage{hyperref}       %
\usepackage{url}            %
\usepackage{booktabs}       %
\usepackage{amsfonts}       %
\usepackage{nicefrac}       %
\usepackage{microtype}      %
\usepackage{xcolor}         %

\usepackage{amsmath}
\usepackage{amssymb}
\usepackage{bm}
\usepackage{csquotes}
\usepackage{colortbl}
\usepackage{caption}
\usepackage{subcaption}
\usepackage{amsthm}
\usepackage{wrapfig}
\usepackage{graphicx}
\usepackage[square,numbers]{natbib}
\usepackage{algorithm}
\usepackage[noend]{algpseudocode}

\newtheorem{theorem}{Theorem}

\newcommand{\principle}{VSML}
\newcommand{\principlefull}{Variable Shared Meta Learning}

\newcommand{\sapproach}{VSML RNN}
\newcommand{\approach}{VSML RNNs}
\newcommand{\sapproachfull}{VSML RNN}
\newcommand{\approachfull}{VSML RNNs}
\newcommand{\fws}{FWPs / LLRs}
\newcommand{\metarnn}{Meta RNN}
\newcommand{\metarnns}{Meta RNNs}

\DeclareMathOperator{\R}{\mathbb{R}}

\usepackage{tikz}
\usetikzlibrary{arrows}
\usetikzlibrary{shapes}
\newcommand*\circled[1]{\tikz[baseline=(char.base)]{
            \node[shape=circle,draw,inner sep=0.5pt] (char) {#1};}}

\newcommand{\fmsg}{f_{\overrightarrow m}}
\newcommand{\bmsg}{f_{\overleftarrow m}}
\newcommand{\sfmsg}{\overrightarrow m}
\newcommand{\sbmsg}{\overleftarrow m}

\newcommand{\intRNN}{\circled{A}}
\newcommand{\intMP}{\circled{B}}
\newcommand{\intNN}{\circled{C}}

\definecolor{w}{RGB}{148, 33, 146}
\definecolor{b}{RGB}{148, 33, 146}
\definecolor{x}{RGB}{0, 118, 186}
\definecolor{y}{RGB}{29, 177, 0}
\definecolor{e}{RGB}{238, 34, 12}
\definecolor{eo}{RGB}{255, 147, 0}
\definecolor{label}{RGB}{0, 0, 255}
\definecolor{prediction}{RGB}{55, 126, 34}

\title{Meta Learning Backpropagation And Improving It}

\author{%
  Louis Kirsch$^{1}$,  J\"urgen Schmidhuber$^{1,2}$\\
  $^1$The Swiss AI Lab IDSIA, University of Lugano (USI) \& SUPSI, Lugano, Switzerland \\
  $^2$King Abdullah University of Science and Technology (KAUST), Thuwal, Saudi Arabia \\
  \texttt{\{louis, juergen\}@idsia.ch} \\
}

\begin{document}

\maketitle

\begin{abstract}
Many concepts have been proposed for meta learning with neural networks (NNs), e.g., NNs that learn to reprogram fast weights, Hebbian plasticity, learned learning rules, and meta recurrent NNs.
Our \emph{Variable Shared Meta Learning (VSML)}
unifies the above and demonstrates that simple weight-sharing and sparsity in an NN is sufficient to express powerful learning algorithms (LAs) in a reusable fashion.
A simple implementation of VSML where the weights of a neural network are replaced by tiny LSTMs allows for implementing the backpropagation LA solely by running in forward-mode.
It can even meta learn new LAs that differ from online backpropagation and generalize to datasets outside of the meta training distribution without explicit gradient calculation.
Introspection reveals that our meta learned LAs learn through fast association in a way that is qualitatively different from gradient descent.
\end{abstract}

\section{Introduction}\label{introduction}

The shift from standard machine learning to meta learning involves learning the learning algorithm (LA) itself, reducing the burden on the human designer to craft useful learning algorithms~\citep{Schmidhuber:87}.
Recent meta learning has primarily focused on generalization from training tasks to similar test tasks, e.g., few-shot learning~\citep{Finn2017}, or from training environments to similar test environments~\citep{Houthooft2018}.
This contrasts human-engineered LAs that generalize across a wide range of datasets or environments.
Without such generalization, meta learned LAs can not entirely replace human-engineered variants.
Recent work demonstrated that meta learning can also successfully generate more general LAs that generalize across wide spectra of environments~\citep{kirsch2020improving,alet2020metalearning,Oh2020}, e.g., from toy environments to Mujoco and Atari.
Unfortunately, however, many recent approaches still rely on a large number of human-designed and unmodifiable inner-loop components such as backpropagation.

Is it possible to implement modifiable versions of backpropagation or related algorithms as part of the end-to-end differentiable activation dynamics of a neural net (NN), instead of inserting them as separate fixed routines?
Here we propose the {\principlefull} (\principle) principle for this purpose.
It introduces a novel way of using sparsity and weight-sharing in NNs for meta learning.
We build on the arguably simplest neural meta learner, the meta recurrent neural network (Meta RNN)~\citep{Hochreiter2001,Duan2016,Wang2016}, by replacing the weights of a neural network with tiny LSTMs.
The resulting system can be viewed as many RNNs passing messages to each other, or as one big RNN with a sparse shared weight matrix, or as a system learning each neuron's functionality and its LA.
{\principle} generalizes the principle behind end-to-end differentiable fast weight programmers~\citep{Schmidhuber1992a,Schmidhuber1993a,Ba2016,Schlag}, hyper networks~\citep{Ha2016}, learned learning rules~\citep{Bengio1992OnRule,Gregor2020,Randazzo2020}, and hebbian-like synaptic plasticity~\citep{Schmidhuber:91fastweights,Schmidhuber1993a,Miconi2018DifferentiableBackpropagation,Miconia,Najarro2020}.
Our mechanism, {\principle}, can implement backpropagation solely in the forward-dynamics of an RNN.
Consequently, it enables meta-optimization of backprop-like algorithms.
We envision a future where novel methods of credit assignment can be meta learned while still generalizing across vastly different tasks.
This may lead to improvements in sample efficiency, memory efficiency, continual learning, and others.
As a first step, our system meta learns online LAs from scratch that frequently learn faster than gradient descent and generalize to datasets outside of the meta training distribution (e.g., from MNIST to Fashion MNIST).
{\principle} is the first neural meta learner without hard-coded backpropagation that shows such strong generalization.

\section{Background}\label{sec:background}

Deep learning-based meta learning that does not rely on fixed gradient descent in the inner loop has historically fallen into two categories,
1) Learnable weight update mechanisms that allow for changing the parameters of an NN to implement a learning rule (\fws), and
2) Learning algorithms implemented in black-box models such as recurrent neural networks (\metarnns).

\paragraph{Fast weight programmers \& Learned learning rules (\fws)}
In a standard NN, the weights are updated by a fixed LA.
This framework can be extended to meta learning by defining an explicit architecture that allows for modifying these weights.
This weight-update architecture augments a standard NN architecture.
NNs that generate or change the weights of another or the same NN are known as fast weight programmers (FWPs)~\citep{Schmidhuber:91fastweights,Schmidhuber1992a,Schmidhuber1993a,Ba2016,Schlag}, hypernetworks~\citep{Ha2016}, NNs with synaptic plasticity~\citep{Miconi2018DifferentiableBackpropagation,Miconia,Najarro2020} or learned learning rules~\citep{Bengio1992OnRule,Gregor2020,Randazzo2020}.
Often these architectures make use of local Hebbian-like update rules or outer-products, and we summarize this category as {\fws}.
In {\fws} the variables $V_L$ that are subject to learning are the weights of the network, whereas the meta-variables $V_M$ that implement the LA are defined by the weight-update architecture.
Note that the dimensionality of $V_L$ and $V_M$ can be defined independently of each other and often $V_M$ are reused in a coordinate-wise fashion for $V_L$ resulting in $|V_L| \gg |V_M|$, where $|\cdot|$ is the number of elements.

\paragraph{Black-box learning in activations (\metarnns)}
It was shown that an RNN such as LSTM can learn to implement an LA~\citep{Hochreiter2001} when the reward or error is given as an input~\citep{Schmidhuber:93selfreficann}.
After meta training, the LA is encoded in the weights of this RNN and determines learning during meta testing.
The activations serve as the memory used for the LA solution.
We refer to this as Meta RNNs~\citep{Hochreiter2001,Duan2016,Wang2016}  (Also sometimes referred to as memory-based meta learning.).
They are conceptually simpler than {\fws} as no additional weight-update rules with many degrees of freedom need to be defined.
In {\metarnns} $V_L$ are the RNN activations and $V_M$ are the parameters for the RNN.
Note that an RNN with $N$ neurons will yield $|V_L| = O(N)$ and $|V_M| = O(N^2)$~\citep{Schmidhuber1993a}.
This means that the LA is largely overparameterized whereas the available memory for learning is very small, making this approach prone to overfitting~\citep{kirsch2020improving}.
As a result, the RNN parameters often encode task-specific solutions instead of generic LAs.
Meta learning a simple and generalizing LA would benefit from $|V_L| \gg |V_M|$.
Previous approaches have tried to mend this issue by adding architectural complexity through additional memory mechanisms~\citep{sun1991neural,mozer1993connectionist,Santoro,Mishra,schlag2021learning}.

\section{{\principlefull} (\principle)}\label{sec:approach}

In {\principle} we build on the simplicity of {\metarnns} while retaining $|V_L| \gg |V_M|$ from {\fws}.
We do this by reusing the same few parameters $V_M$ many times in an RNN (via variable sharing) and introducing sparsity in the connectivity.
This yields several interpretations for {\principle}:
\begin{enumerate}
    \item[\intRNN] \textbf{VSML as a single {\metarnn} with a sparse shared weight matrix (\autoref{fig:as-mrnn}).}
    The most general description. 
    \item[\intMP] \textbf{VSML as message passing between RNNs (\autoref{fig:vsmrnn}).}
    We choose a simple sharing and sparsity scheme for the weight matrix such that it corresponds to multiple RNNs with shared parameters that exchange information.
    \item[\intNN] \textbf{VSML as complex neurons with learned updates (\autoref{fig:as-nn}).}
    When choosing a specific connectivity between RNNs, states / activations $V_L$ of these RNNs can be interpreted as the weights of a conventional NN, consequently blurring the distinction between a weight and an activation.
\end{enumerate}

\begin{figure*}[t]
    \centering
    \hfill
    \begin{subfigure}[b]{0.25\textwidth}
        \centering
        \includegraphics[width=0.75\textwidth]{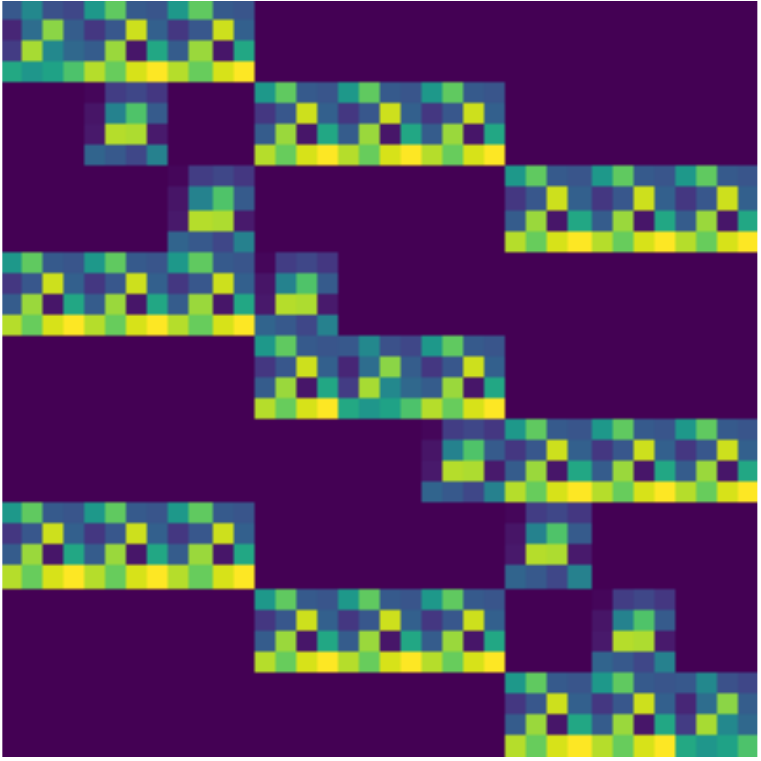}
        \caption{Viewed as a single RNN (structured weight matrix)}
        \label{fig:as-mrnn}
    \end{subfigure}
    \hfill
    \begin{subfigure}[b]{0.25\textwidth}
        \centering
        \includegraphics[width=0.8\textwidth]{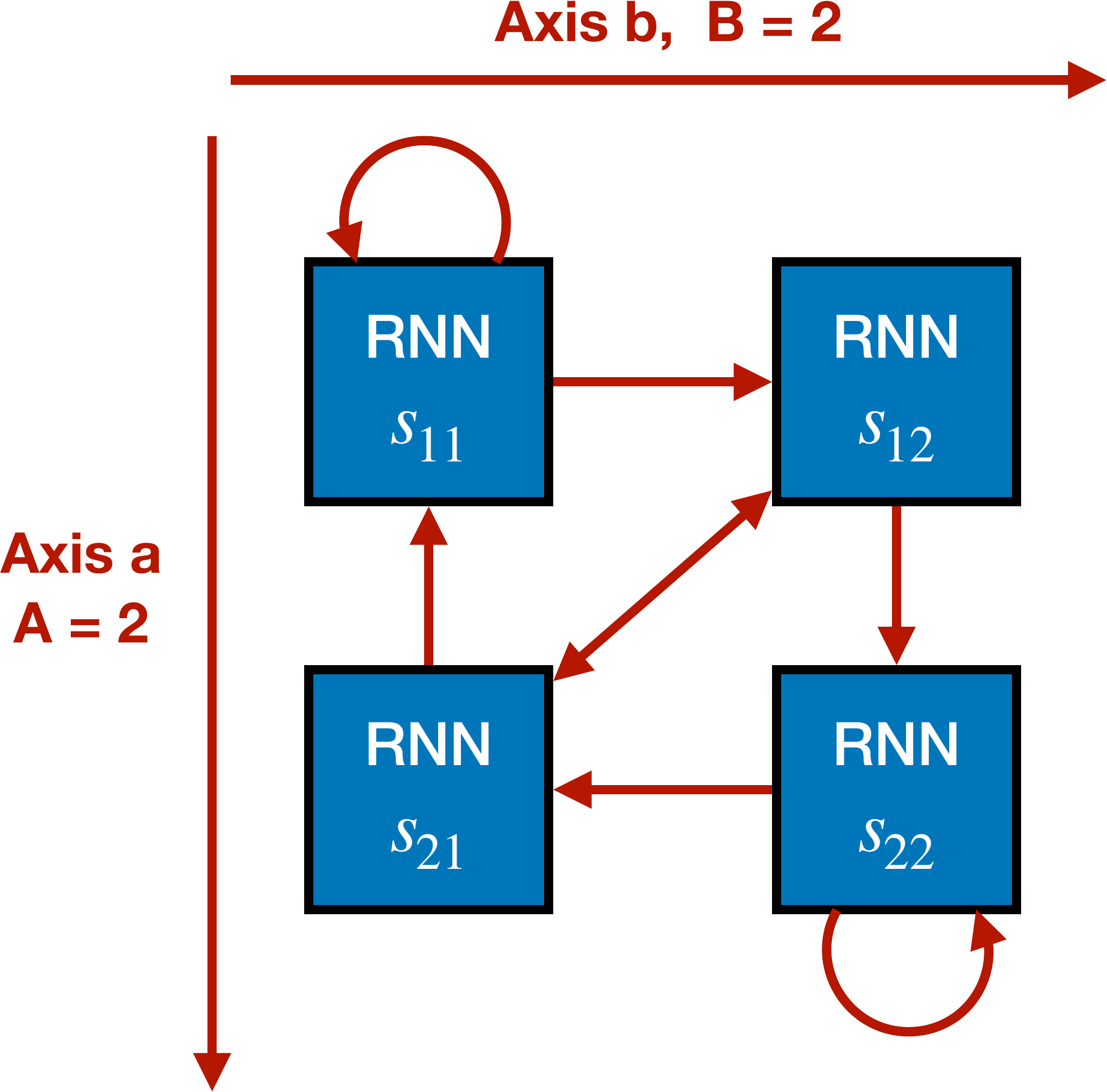}
        \caption{One {\sapproach} = many sub-RNNs}
        \label{fig:vsmrnn}
    \end{subfigure}
    \hfill
    \begin{subfigure}[b]{0.40\textwidth}
        \centering
        \includegraphics[width=0.95\textwidth]{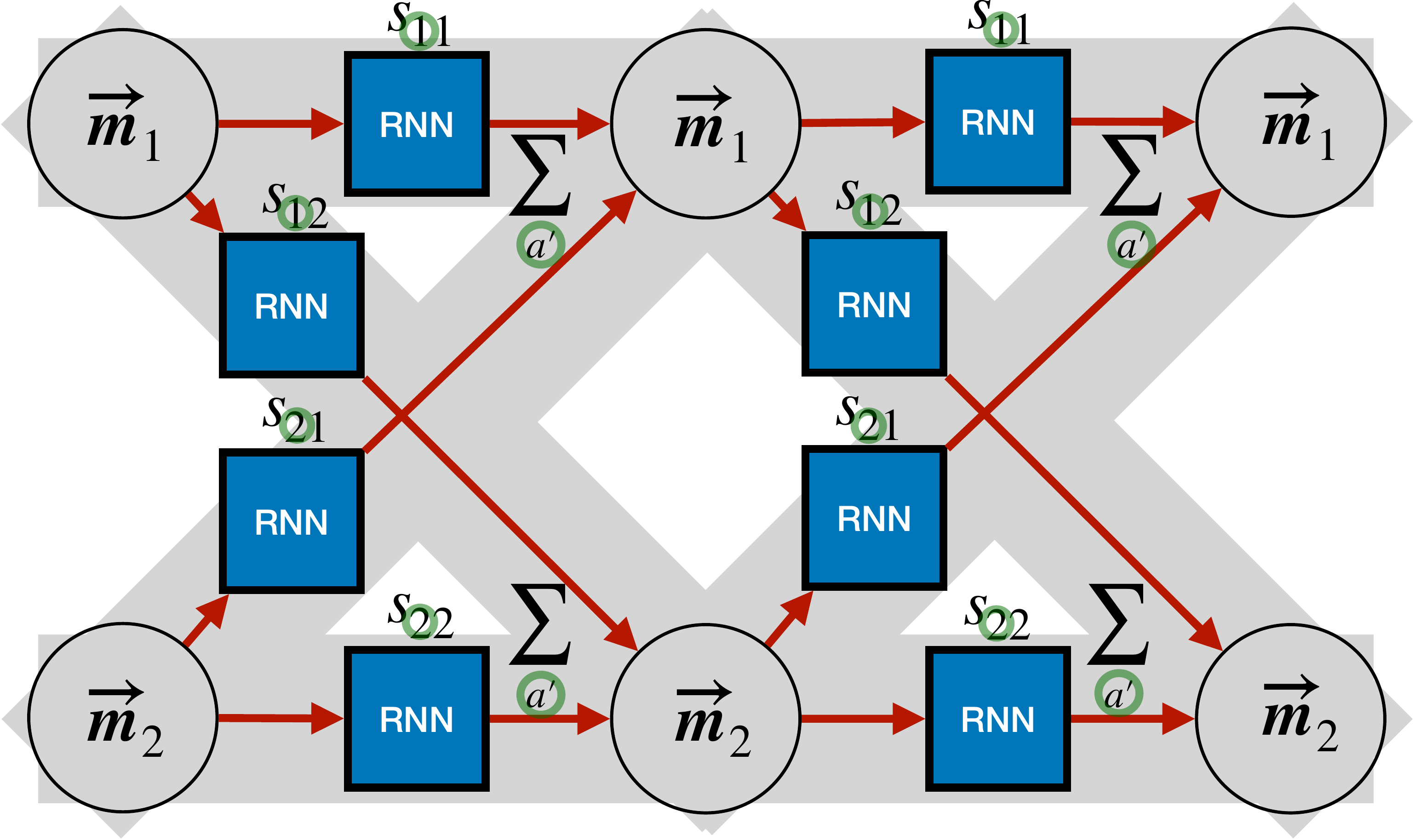}
        \caption{Viewed as an NN with complex neurons}
        \label{fig:as-nn}
    \end{subfigure}
    \hfill
    \caption{Different perspectives on {\principle}:
    (a) A single {\metarnn}~\citep{Hochreiter2001} where entries in the weight matrix are shared or zero.
    (a)
    {\principle} consists of many sub-RNNs with shared parameters $V_M$ passing messages between each other.
    (c) {\principle} implements an NN with complex neurons (here 2 neurons).
    $V_M$ determines the nature of weights, how these are used in the neural computation, and the LA by which those are updated.
    Each weight $w_{ab} \in \R$ is represented by the multi-dimensional RNN state $s_{ab} \in \R^N$.
    Neuron activations correspond to messages $\protect\sfmsg$ passed between sub-RNNs.}
    \label{fig:vsmrnn-perspectives}
\end{figure*}

\paragraph{Introducing variable sharing to Meta RNNs}

We begin by formalizing {\metarnns} which often use multiplicative gates such as the  LSTM~\citep{Gers:2000nc,Hochreiter:97lstm} or its variant GRU~\citep{cho14GRU}.
For notational simplicity, we consider a vanilla RNN.
Let $s \in \R^N$ be the hidden state of an RNN.
The update for an element $j \in \{1, \ldots, N\}$ is given by
\begin{equation}\label{eq:meta_rnn}
    s_j \leftarrow f_{\textrm{RNN}}(s)_j = \sigma(\sum_i s_i W_{ij}),
\end{equation}
where $\sigma$ is a non-linear activation function, $W \in \R^{N \times N}$, and the bias is omitted for simplicity.
We also omit inputs by assuming a subset of $s$ to be externally provided.
Each application of \autoref{eq:meta_rnn} reflects a single time tick in the RNN.

We now introduce variable sharing (reusing $W$) into the RNN by duplicating the computation along two axes of size $A, B$ (here $A = B$, which will later be relaxed) giving $s \in \R^{A \times B \times N}$.
For $a \in \{1, \ldots, A\}, b \in \{1, \ldots, B\}$ we have
\begin{equation}
    s_{abj} \leftarrow f_{\textrm{RNN}}(s_{ab})_j = \sigma(\sum_i s_{abi} W_{ij}).
\end{equation}
This can be viewed as multiple RNNs arranged on a 2-dimensional grid, with shared parameters that update independent states.
Here, we chose a particular arrangement (two axes) that will facilitate the interpretation {\intNN} of RNNs as weights.

\paragraph{VSML as message passing between RNNs}
The computation so far describes $A \cdot B$ independent RNNs.
We connect those by passing messages (interpretation {\intMP})
\begin{equation}\label{eq:message_passing}
    s_{ab} \leftarrow f_{\textrm{RNN}}(s_{ab}, \sfmsg_a),
\end{equation}
where the message $\sfmsg_b = \sum_{a'} \fmsg(s_{a'b})$ with $b \in \{1,\ldots,A=B\}$, $\fmsg: \R^{N} \to \R^{N'}$ is fed as an additional input to each RNN.
This is related to Graph Neural Networks~\citep{sperduti1994encoding,wu2020comprehensive}.
Summing over the axis $A$ (elements $a'$) corresponds to an RNN connectivity mimicking those of weights in an NN (to facilitate interpretation {\intNN}).
We emphasise that other schemes based on different kinds of message passing and graph connectivity are possible.
For a simple $\fmsg$ defined by the matrix $C \in \R^{N \times N}$, we may equivalently write
\begin{equation}\label{eq:interact}
    s_{abj} \leftarrow \sigma ( \sum_i s_{abi} W_{ij} + \sum_{a'} \fmsg(s_{a'a})_j)
    = \sigma ( \sum_i s_{abi} W_{ij} + \sum_{a',i} s_{a'ai} C_{ij}).
\end{equation}
This constitutes the minimal version of {\principle} with $V_M := (W, C)$ and is visualized in \autoref{fig:vsmrnn}.

\paragraph{VSML as a Meta RNN with a sparse shared weight matrix}

It is trivial to see that with $A = 1$ and $B = 1$ we obtain a single RNN and \autoref{eq:interact} recovers the original Meta RNN \autoref{eq:meta_rnn}.
In the general case, we can derive an equivalent formulation that corresponds to a single standard RNN with a single matrix $\tilde W$ that has entries of zero and shared entries
\begin{equation}\label{eq:sparse-shared-equiv}
    s_{abj} \leftarrow \sigma(\sum_{c,d,i} s_{cdi} \tilde W_{cdiabj}),
\end{equation}
where the six axes can be flattened to obtain the two axes.
For \autoref{eq:interact} and \autoref{eq:sparse-shared-equiv} to be equivalent, $\tilde W$ must satisfy (derivation in \autoref{sec:rnn-equiv-derivation})
\begin{equation}\label{eq:sparse-shared}
    \tilde W_{cdiabj} = \begin{cases}
        C_{ij}, & \text{if } d = a \land ( d \neq b \lor c \neq a ). \\
        W_{ij}, & \text{if } d \neq a \land d = b \land c = a. \\
        C_{ij} + W_{ij}, & \text{if } d = a \land d = b \land c = a. \\
        0, & \text{otherwise.}
    \end{cases}
\end{equation}
This corresponds to interpretation {\intRNN} with the weight matrix visualized in \autoref{fig:as-mrnn}.
To distinguish between the single sparse shared RNN and the connected RNNs, we now call the latter \emph{sub-RNNs}.

\paragraph{VSML as complex neurons with learned updates}

The arrangement and connectivity of the sub-RNNs as described in the previous paragraphs corresponds to that of weights in a standard NN.
Thus, in interpretation {\intNN}, {\principle} can be viewed as defining complex neurons where each sub-RNN corresponds to a weight in a standard NN as visualized in \autoref{fig:as-nn}.
All these sub-RNNs share the same parameters but have distinct states.
The current formulation corresponds to a single NN layer that is run recurrently.
We will generalize this to other architectures in the next section.
$A$ corresponds to the dimensionality of the inputs and $B$ to that of the outputs in that layer.

\begin{wrapfigure}{r}{0.7\textwidth}
    \vspace{-3mm}
    \centering
    \includegraphics[width=0.7\textwidth]{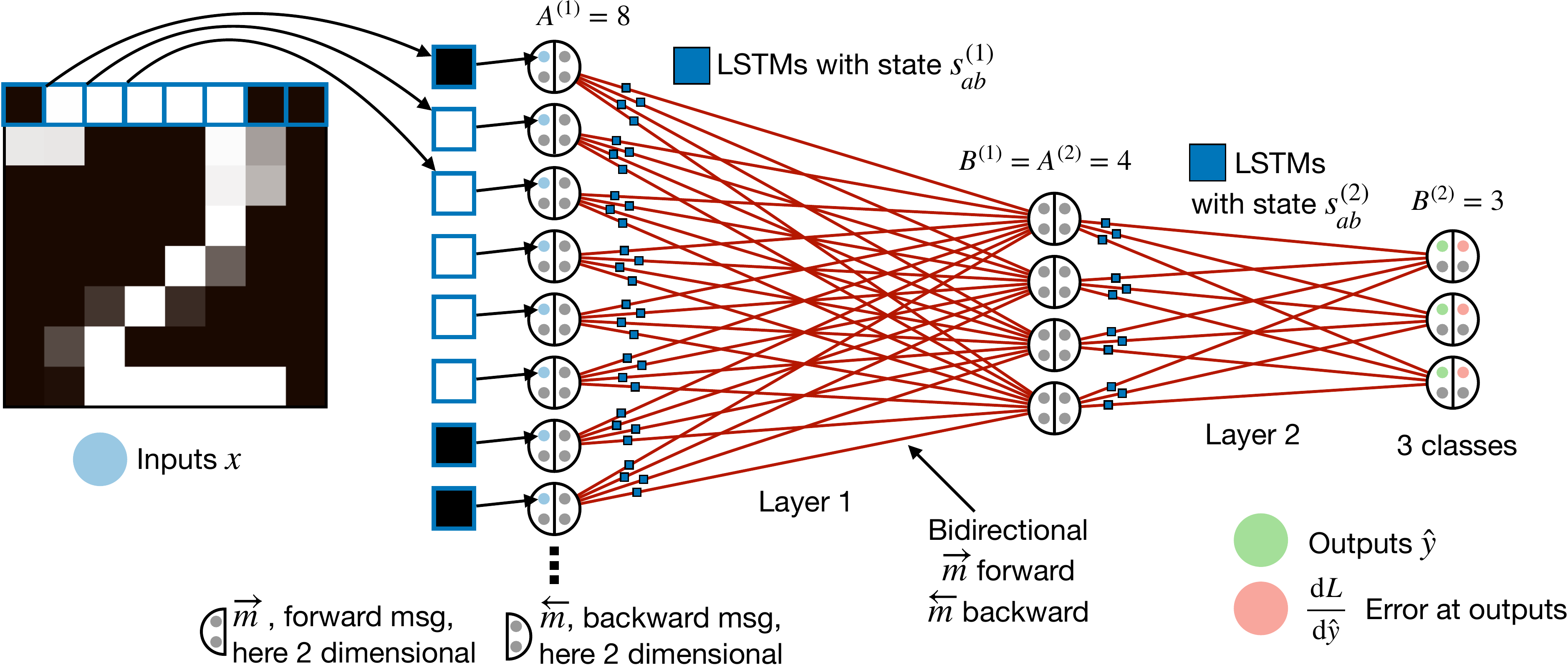}
    \caption{
    The neural interpretation of {\principle} replaces all weights of a standard NN with tiny LSTMs using shared parameters (resembling complex neurons).
    This allows these LSTMs to define both the neural forward computation as well as the learning algorithm that determines how the network is updated.
    Information flows forward and backward in the network through multi-dimensional messages $\protect\sfmsg$ and $\protect\sbmsg$, generalizing the dynamics of an NN trained using backpropagation.}
    \label{fig:vsml-mnist-arch}
    \vspace{-3mm}
\end{wrapfigure}
The role of weights in a standard neural network is now assigned to the states of RNNs.
This allows these RNNs to define both the neural forward computation as well as the learning algorithm that determines how the network is updated (where the mechanism is shared across the network).
In the case of backpropagation, this would correspond to the forward and backward passes being implemented purely in the recurrent dynamics.
We will demonstrate the practical feasibility of this in \autoref{sec:backprop}.
The emergence of RNN states as weights quickly leads to confusing terminology when RNNs have `meta weights'.
Instead, we simply refer to meta variables $V_M$ and learned variables $V_L$.

With this interpretation, {\principle} can be seen as a generalization of learned learning rules~\citep{Bengio1992OnRule,Gregor2020,Randazzo2020} and Hebbian-like differentiable mechanisms or fast weights more generally~\citep{Schmidhuber:91fastweights,Schmidhuber1993a,Miconi2018DifferentiableBackpropagation,Miconia} where RNNs replace explicit weight updates.
In standard NNs, weights and activations have multiplicative interactions.
For {\approach} to mimic such computation we require multiplicative interactions between parts of the state $s$.
Fortunately, LSTMs already incorporate this through gating and can be directly used in place of RNNs.

\paragraph{Stacking {\approach} and feeding inputs}

To get a structure similar to one of the non-recurrent deep feed-forward architectures (FNNs), we stack multiple {\approach} where their states are untied and their parameters are tied.%
\footnote{The resultant architecture as a whole is still recurrent.
Note that even standard FNNs are recurrent if the LA (backpropagation) is taken into account.}
This is visualized with two layers in \autoref{fig:vsml-mnist-arch} where the states $s^{(2)}$ of the second column of sub-RNNs are distinct from the first column $s^{(1)}$.
The parameters $A^{(k)}$ and $B^{(k)}$ describing layer sizes can then be varied for each layer $k \in \{1, \ldots, K\}$ constrained by $A^{(k)} = B^{(k-1)}$.
The updated \autoref{eq:message_passing} with distinct layers $k$ is given by $s_{ab}^{(k)} \leftarrow f_{\textrm{RNN}}(s_{ab}^{(k)}, \sfmsg_a^{(k)})$ where $\sfmsg_b^{(k+1)} := \sum_{a'} \fmsg(s_{a'b}^{(k)})$ with $b \in \{1,\ldots,B^{(k)}=A^{(k+1)}\}$.
\begin{wrapfigure}{r}{0.5\textwidth}
    \centering
    \includegraphics[width=0.5\textwidth]{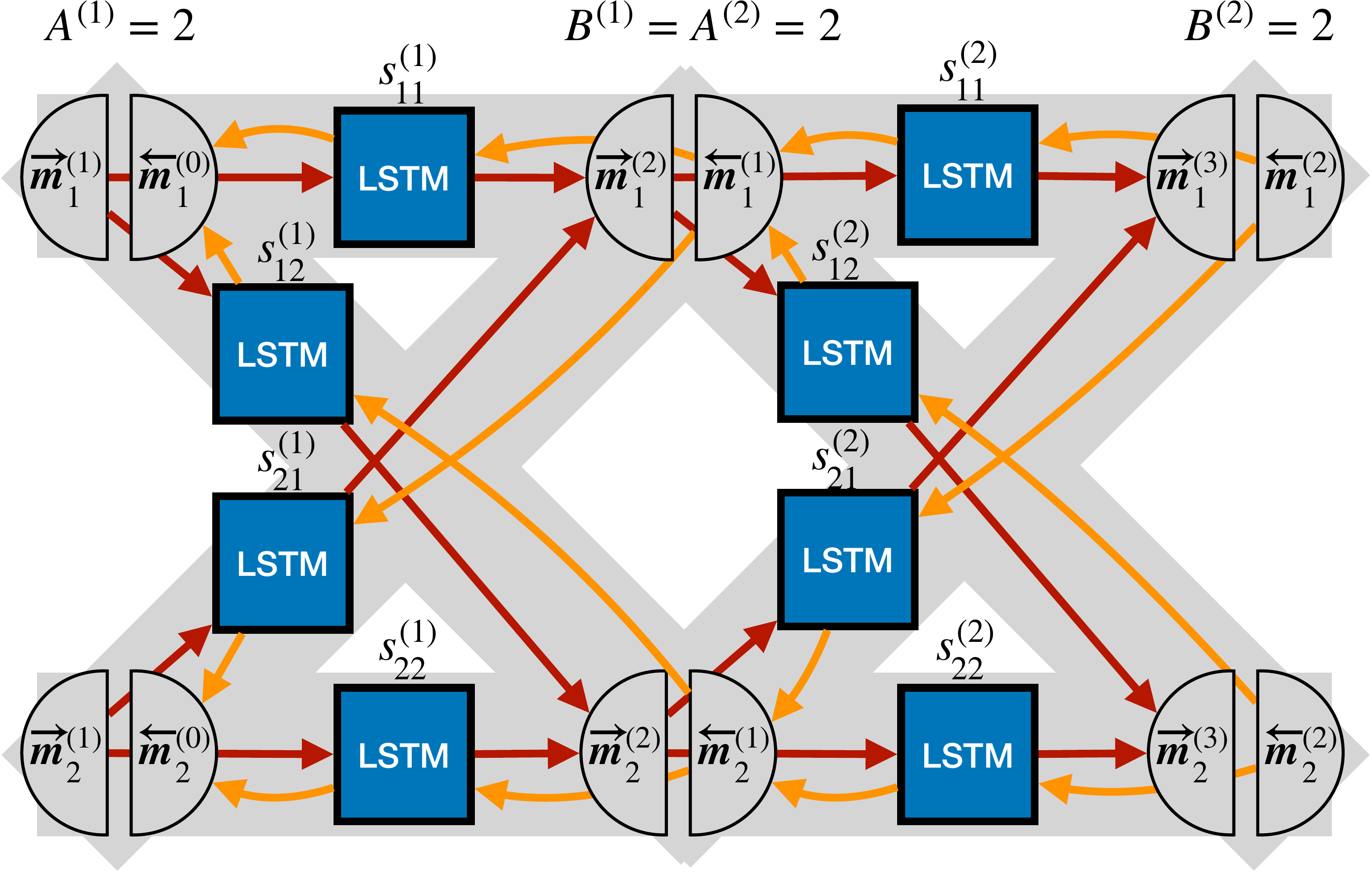}
    \caption{VSML with forward messages $\protect\sfmsg$ and backward messages $\protect\sbmsg$ to form a two-layer NN with shared LSTM parameters but distinct states.}
    \label{fig:vsml-fwd-bwd}
    \vspace{-5mm}
\end{wrapfigure}
To prevent information from flowing only forward in the network, we use an additional backward message
\begin{equation}\label{eq:message-forward-backward}
s_{ab}^{(k)} \leftarrow f_{RNN}(s_{ab}^{(k)}, \sfmsg_a^{(k)}, \sbmsg_b^{(k)}),
\end{equation}
where $\sbmsg_a^{(k-1)} := \sum_{b'} \bmsg(s_{ab'}^{(k)})$ with $a \in \{1,\ldots,A^{(k)}=B^{(k-1)}\}$ (visualized in \autoref{fig:vsml-fwd-bwd}).
The backward transformation is given by $\bmsg: \R^N \to \R^{N''}$.

Similarly, other neural architectures can be explicitly constructed (e.g. convolutional NNs, \autoref{sec:cnns}).
Some architectures may be learned implicitly if positional information is fed into each sub-RNN (\autoref{sec:training_details}).
We then update all states $s^{(k)}$ in sequence $1, \ldots, K$ to mimic sequential layer execution.
We may also apply multiple RNN ticks for each layer $k$.

To provide the {\sapproach} with data, each time we execute the operations of the first layer, a single new datum $x \in \R^{A(1)}$ (e.g. one flattened image) is distributed across all sub-RNNs.
In our present experiments, we match the axis $A(1)$ to the input datum dimensionality such that each dimension (e.g., pixel) is fed to different RNNs.
This corresponds to initializing the forward message $\sfmsg_{a1}^{(1)} := x_a$ (padding $\sfmsg$ with zeros if necessary).
Similarly, we read the output $\hat y \in \R^{B(K)}$ from $\hat y_a := \sfmsg_{a1}^{(K+1)}$.
Finally, we feed the error $e \in \R^{B(K)}$ at the output such that $\sbmsg_{b1}^{(K)} := e_b$.
See \autoref{fig:vsml-mnist-arch} for a visualization.
Alternatively, multiple input or output dimensions could be patched together and fed into fewer sub-RNNs.

\subsection{Meta learning general-purpose learning algorithms from scratch}\label{sec:meta-scratch}

Having formalized {\principle}, we can now use end-to-end meta learning to create LAs from scratch in \autoref{alg:meta_training}.
We simply optimize the LSTM parameters $V_M$ to minimize the sum of prediction losses over many time steps starting with random states $V_L := \{s_{ab}^{(k)}\}$.
We focus on meta learning online LAs where one example is fed at a time as done in Meta RNNs~\citep{Hochreiter2001,Wang2016,Duan2016}.
Meta training may be performed using end-to-end gradient descent or gradient-free optimization such as evolutionary strategies~\citep{Wierstra2011,Salimans2017}.
The latter is significantly more efficient on VSML compared to standard NNs due to the small parameter space $V_M$.
Crucially, during meta testing, no explicit gradient descent is used.

\begin{algorithm}[H]
    \centering	
    \begin{algorithmic}	
        \Require Dataset(s) $D = \{(x_i, y_i)\}$
        \State $V_M \leftarrow$ initialize LSTM parameters
        \While{meta loss has not converged} \Comment{Outer loop in parallel over datasets $D$}
            \State $V_L = \{s_{ab}^{(k)}\} \leftarrow$ initialize LSTM states $\quad \forall a,b,k$
            \For{$(x, y) \in \{(x_1, y_1), \ldots, (x_T, y_T)\} \subset D$} \Comment{Inner loop over $T$ examples}
                \State $\sfmsg^{(1)}_{a1} := x_a \quad \forall a$ \Comment{Initialize from input image x}
                \For{$k \in \{1, \ldots, K\}$} \Comment{Iterating over $K$ layers}
                    \State $s_{ab}^{(k)} \leftarrow f_{RNN}(s_{ab}^{(k)}, \sfmsg_a^{(k)}, \sbmsg_b^{(k)}) \quad \forall a,b$ \Comment{\autoref{eq:message-forward-backward}}
                    \State $\sfmsg_b^{(k+1)} := \sum_{a'} \fmsg(s_{a'b}^{(k)}) \quad \forall b$ \Comment{Create forward message}
                    \State $\sbmsg_a^{(k-1)} := \sum_{b'} \bmsg(s_{ab'}^{(k)}) \quad \forall a$ \Comment{Create backward message}
                \EndFor
                \State $\hat y_a := \sfmsg_{a1}^{(K+1)} \quad \forall a$ \Comment{Read output}
                \State $e := \nabla_{\hat y} L(\hat y, y)$ \Comment{Compute error at outputs using loss $L$}
                \State $\sbmsg_{b1}^{(K)} := e_b \quad \forall b$ \Comment{Input errors}
            \EndFor
            \State $V_M \leftarrow V_M - \alpha \nabla_{V_M} \sum_{t=1}^{T} L(\hat y(t), y(t))$, obtaining $\nabla_{V_M}$ either by
            \begin{itemize}
                \item back-propagation through the inner loop
                \item evolution strategies, using a search distribution $p(V_M)$
            \end{itemize}
        \EndWhile	
    \end{algorithmic}	
    \caption{{\principle}: Meta Training}\label{alg:meta_training}
\end{algorithm}

\subsection{Learning to implement backpropagation in RNNs}\label{sec:backprop}

\begin{wrapfigure}{l}{0.3\textwidth}
    \vspace{-5mm}
    \begin{center}
      \includegraphics[width=3cm]{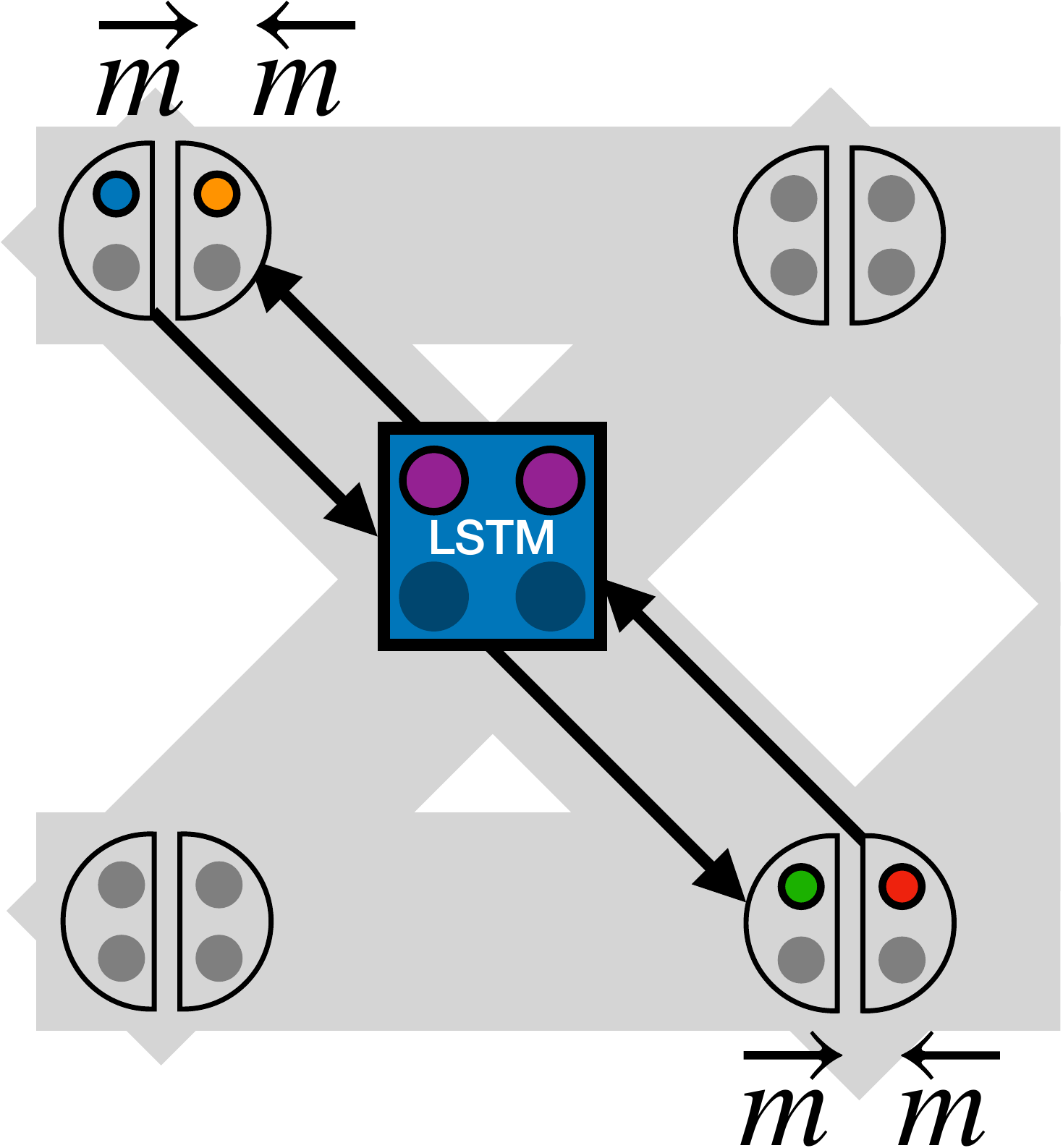}
    \end{center}
    \caption{To implement backpropagation we optimize the {\sapproach} to use and update weights $\color{w} w$ and biases $\color{b} b$ as part of the state $s_{ab}$ in each sub-RNN.
    Inputs are pre-synaptic $\color{x} x$ and error $\color{e} e$.
    Outputs are post-synaptic $\color{y} \hat y$ and error $\color{eo} \hat e'$.
    }
    \label{fig:impl-backprop}
    \vspace{-5mm}
\end{wrapfigure}

An alternative to end-to-end meta learning is to first investigate whether the {\sapproach} can implement backpropagation.
Due to the algorithm's ubiquitous use, it seems desirable to be able to meta learn backpropagation-like algorithms.
Here we investigate how {\approach} can learn to implement backpropagation purely in their recurrent dynamics.
We do this by optimizing $V_M$ to
(1) store a weight $w$ and bias $b$ as a subset of each state $s_{ab}$,
(2) compute $y = \textrm{tanh}(x)w + b$ to implement neural forward computation,
and (3) update $w$ and $b$ according to the backpropagation algorithm~\citep{Linnainmaa:1970}.
We call this process \emph{learning algorithm cloning} and it is visualized in \autoref{fig:impl-backprop}.

We designate an element of each message $\sfmsg_a^{(k)}$,  $\sbmsg_b^{(k)}$, $\fmsg(s_{ab}^{(k)})$, $\bmsg(s_{ab}^{(k)})$ as the input $x$, error $e$, and output $\hat y$ and error $\hat e'$.
Similarly, we set $w := s_{ab1}$ and $b := s_{ab2}$.
We then optimize $V_M$ via gradient descent to regress $\hat y$, $\Delta w$, $\Delta b$, and $\hat e'$ toward their respective targets.
We can either generate the training dataset $D := \{(x, w, b, y, e, e')_i\}$ randomly or run a `shadow' NN on some supervised problem and fit the {\sapproach} to its activations and parameter updates.
Multiple iterations in the {\sapproach} would then correspond to evaluating the network and updating it via backpropagation.
The activations from the forward pass necessary for credit assignment could be memorized as part of the state $s$ or be explicitly stored and fed back.
For simplicity, we chose the latter to clone backpropagation.
We continuously run the {\sapproach} forward, alternately running the layers in order $1, \ldots, K$ and in reverse order $K, \ldots, 1$.%
\footnote{Executing layers in reverse order is not strictly necessary as information always also flows backwards through $\sbmsg$ but makes LA cloning easier.
}

\section{Experiments}

\begin{wrapfigure}{r}{0.5\textwidth}
    \vspace{-10mm}
    \centering
    \includegraphics[width=0.5\textwidth]{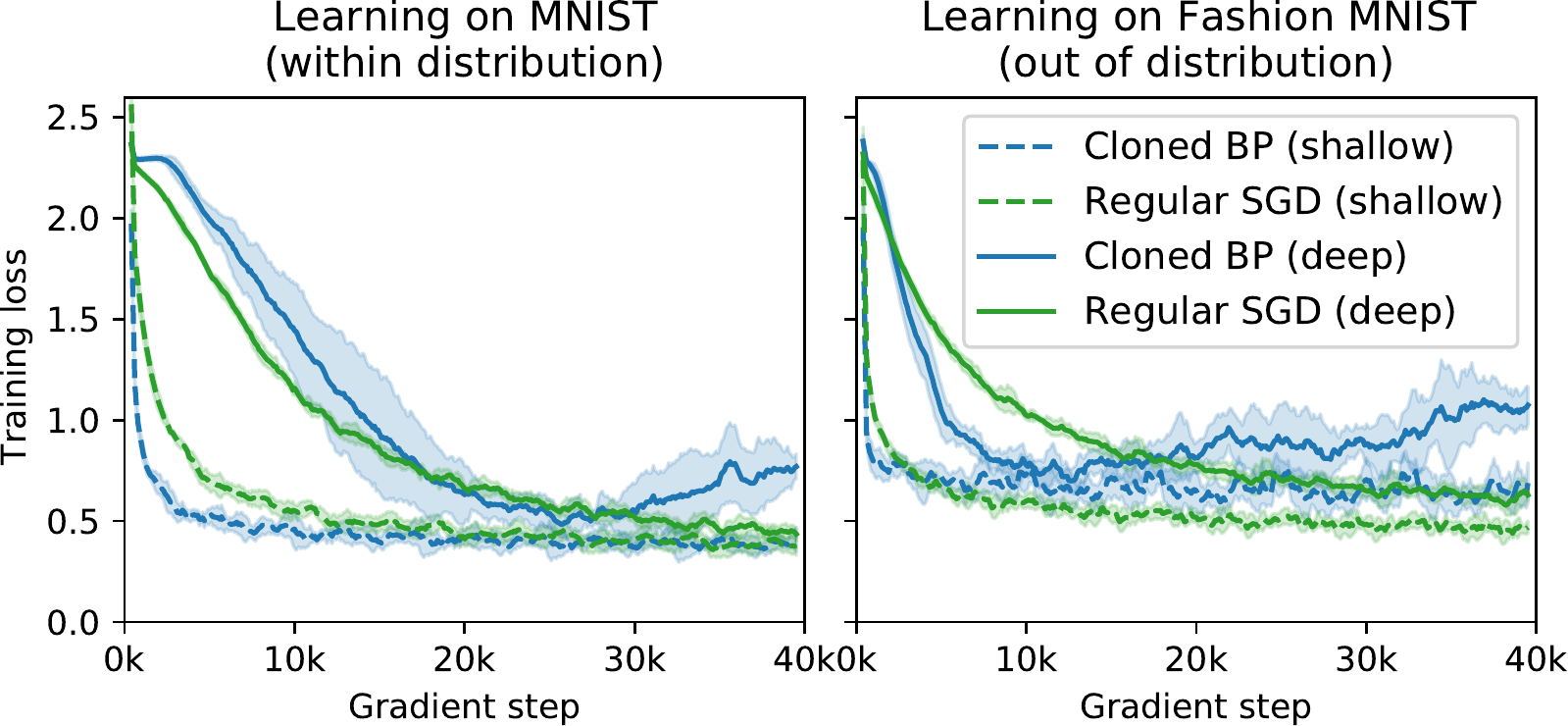}
    \caption{The {\sapproach} is optimized to implement backpropagation in its recurrent dynamics on MNIST, then tested both on MNIST and Fashion MNIST where it performs learning purely by unrolling the LSTM.
    We test on shallow and deep architectures (single hidden layer of 32 units).
    Standard deviations are over 6 seeds.}
    \label{fig:impl_backprop}
    \vspace{-5mm}
\end{wrapfigure}
First, we demonstrate the capabilities of the {\sapproach} by showing that it can implement neural forward computation and backpropagation in its recurrent dynamics on the \emph{MNIST}~\citep{lecun2010mnist} and \emph{Fashion MNIST}~\citep{xiao2017fashionmnist} dataset.
Then, we show how we can meta learn an LA from scratch on one set of datasets and then successfully apply it to another (out of distribution).
Such generalization is enabled by extensive variable sharing where we have very few meta variables $|V_M| \approx 2,400$ and many learned variables $|V_L| \approx 257,200$.
We also investigate the robustness of the discovered LA.
Finally, we introspect the meta learned LA and compare it to gradient descent.

Our implementation uses LSTMs and the message interpretation from \autoref{eq:message-forward-backward}.
Hyperparameters, training details, and additional experiments can be found in the appendix.

\subsection{{\approach} can implement backpropagation}

As described in \autoref{sec:backprop}, we optimize the {\sapproach} to implement backpropagation.
We structure the sub-RNNs to mimic a feed-forward NN with either one hidden layer or no hidden layers.
To obtain training targets, we instantiate a standard NN, the shadow network, and feed it MNIST data.
After cloning, we then run the LA encoded in the {\sapproach} on the MNIST and Fashion MNIST dataset and observe that it performs learning purely in its recurrent dynamics, making explicit gradient calculations unnecessary.
\autoref{fig:impl_backprop} shows the learning curve on these two datasets.
Notably, learning works both on MNIST (within distribution) and on Fashion MNIST (out of distribution).
We observe that the loss is decently minimized, albeit regular gradient descent still outperforms our cloned backpropagation.
This may be due to non-zero errors during learning algorithm cloning, in particular when these errors accumulate in the deeper architecture.
It is also possible that the VSML states (`weights') deviate too far from ranges seen during cloning, in particular in the deep case when the loss starts increasing.
We obtain 87\% (deep) and 90\% (shallow) test accuracy on MNIST and 76\% (deep) and 80\% (shallow) on Fashion MNIST (focusing on successful cloning over performance).

\subsection{Meta learning supervised learning from scratch}

\begin{wrapfigure}{r}{0.5\textwidth}
    \vspace{-5mm}
    \centering
    \includegraphics[width=0.5\columnwidth]{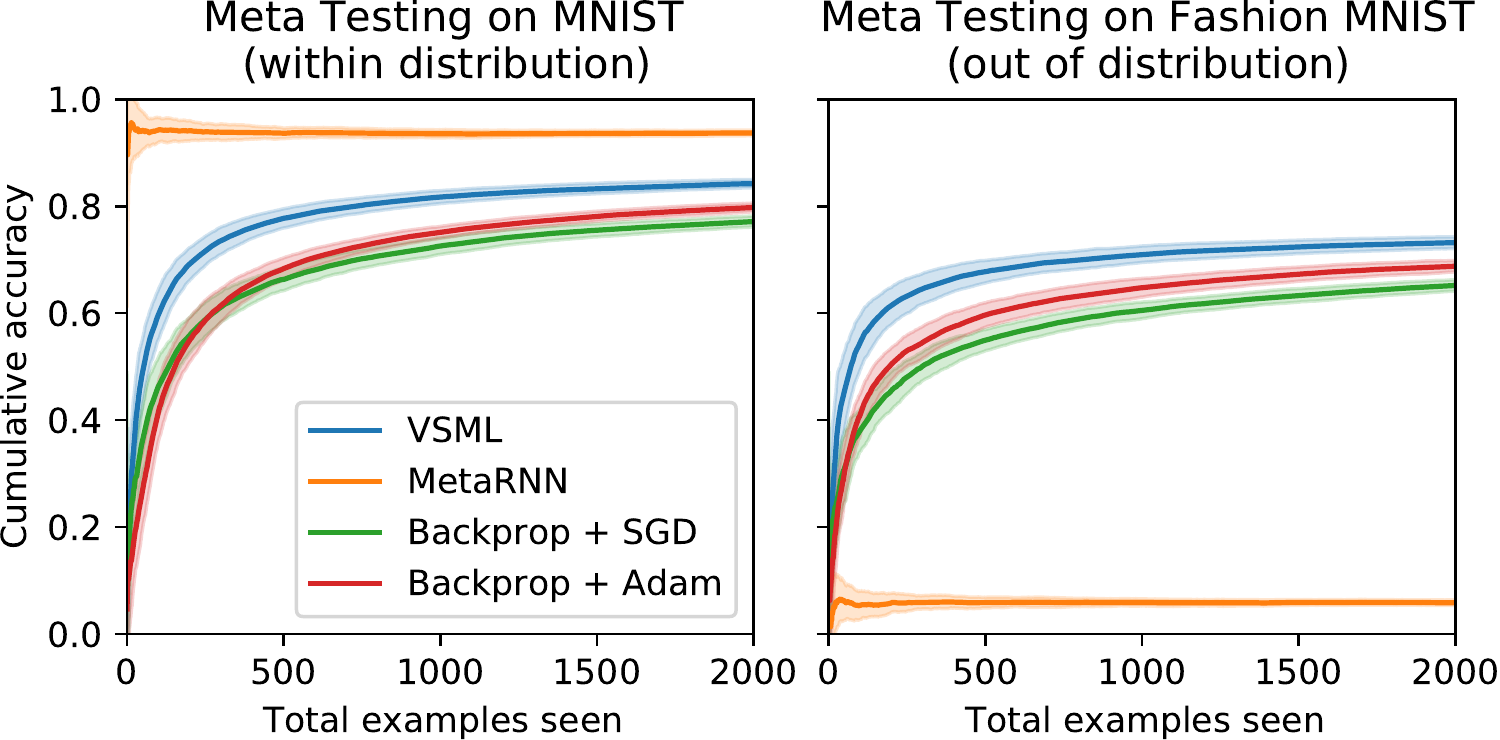}
    \caption{
    The {\sapproach} can be directly meta trained on MNIST to minimize the sum of errors when classifying online starting from a random state initialization.
    This allows for faster learning during meta testing compared to online gradient descent with Adam on the same dataset and even generalizes to a different dataset (Fashion MNIST).
    In comparison, a standard Meta RNN~\citep{Hochreiter2001} strongly overfits in the same setting.
    Standard deviations are over 128 seeds.
    }
    \label{fig:scratch}
    \vspace{-3mm}
\end{wrapfigure}

In the previous experiments, we have established that {\principle} is expressive enough to meta-optimize backpropagation-like algorithms.
Instead of cloning an LA, we now meta learn from scratch as described in \autoref{sec:meta-scratch}.
Here, we use a single layer ($K = 1$) from input to output dimension and run it for two ticks per image with $N = 16$ and $N' = N'' = 8$.
First, the {\sapproachfull} is meta trained end-to-end using evolutionary strategies (ES)~\citep{Salimans2017} on MNIST to minimize the sum of cross-entropies over 500 data points starting from random state initializations.
As each image is unique and $V_M$ can not memorize the data, we are implicitly optimizing the {\sapproach} to generalize to future inputs given all inputs it has seen so far.
We do not pre-train $V_M$ with a human-engineered LA.

During meta testing on MNIST (\autoref{fig:scratch}) we plot the cumulative accuracy on all previous inputs on the y axis ($\frac{1}{T}\sum_{t=1}^{T} c_t$ after example $T$ with binary $c_t$ indicating prediction correctness).
For each example, the prediction when this example was fed to the RNN is used, thus measuring sample efficient learning.
This evaluation protocol is similar to the one used in Meta RNNs~\citep{Wang2016,Duan2016}.
We observe that learning is considerably faster compared to the baseline of online gradient descent (no mini batching, the learning rate appropriately tuned).
One possibility is that {\principle} simply overfits to the training distribution.
We reject this possibility by meta testing the same unmodified RNN on a different dataset, here Fashion MNIST.
Learning still works well, meaning we have meta learned a fairly general LA (although performance at convergence still lacks behind a little).
This generalization is achieved without using any hardcoded gradients during meta testing purely by running the RNN forward.
In comparison to VSML, a {\metarnn} heavily overfits.

\subsection{Robustness to varying inputs and outputs}\label{sec:robustness}

\begin{wrapfigure}{r}{0.6\textwidth}
    \vspace{-8mm}
    \centering
    \includegraphics[width=0.6\columnwidth]{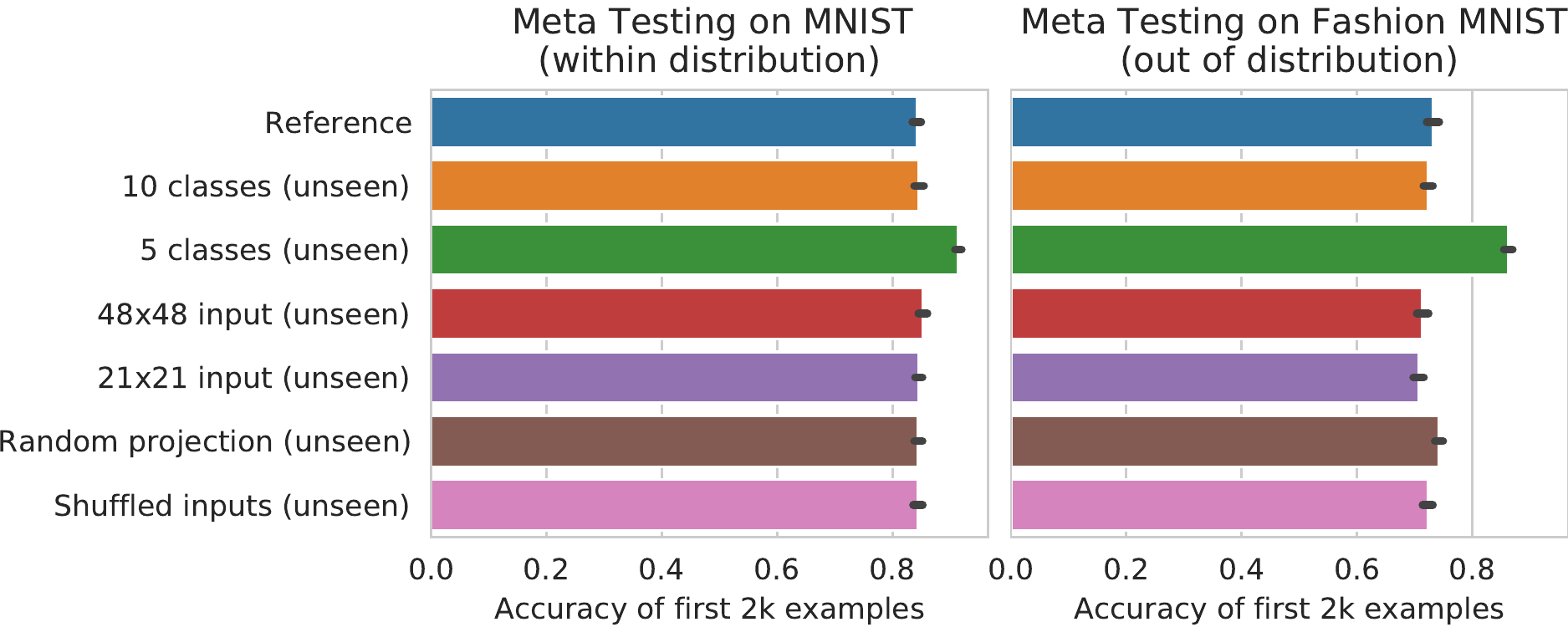}
    \caption{The meta learned learning algorithm is robust to permutations and size changes in the inputs and outputs.
    All six configurations have not been seen during training and perform comparable to the unchanged reference.
    Standard deviations are over 32 seeds.
    }
    \label{fig:invariances}
    \vspace{-3mm}
\end{wrapfigure}

A defining property of {\principle} is that the same parameters $V_M$ can be used to learn on varying input and output sizes.
Further, the architecture and thus the meta learned LA is invariant to the order of inputs and outputs.
In this experiment, we investigate how robust we are to such changes.
We meta train across MNIST with 3, 4, 6, and 7 classes.
Likewise, we train across rescaled versions with 14x14, 28x28, and 32x32 pixels.
We also randomly project all inputs using a linear transformation, with the transformation fixed for all inner learning steps.
In \autoref{fig:invariances} we meta test on 6 configurations that were not seen during meta training.
Performance on all of these configurations is comparable to the unchanged reference from the previous section.
In particular, the invariance to random projections suggests that we have meta learned a learning algorithm beyond transferring learned representations~\citep[compare][]{Finn2017,Triantafillou2019,tseng2020cross}.

\subsection{Varying datasets}

\begin{figure}
    \centering
    \includegraphics[width=1.0\textwidth]{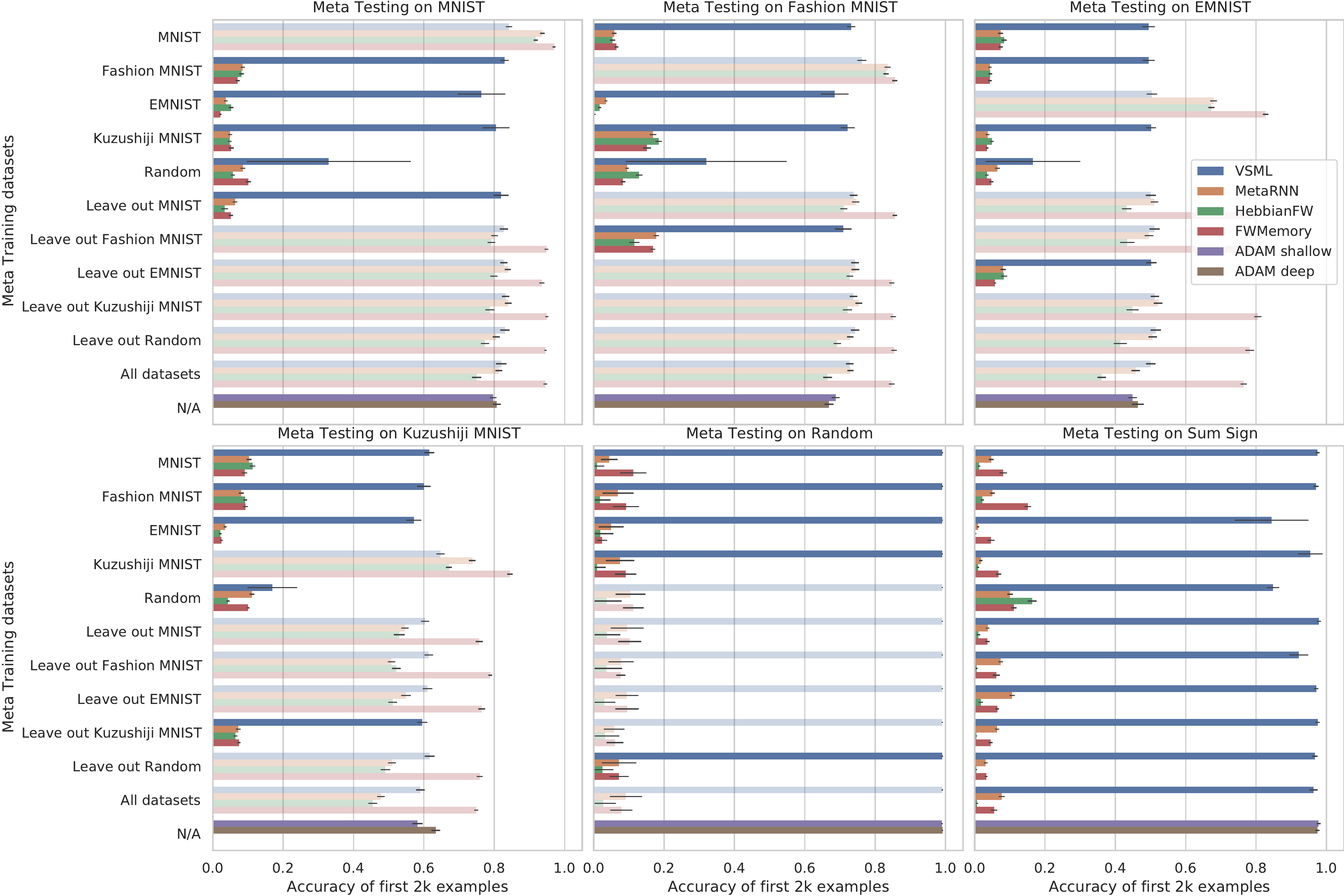}
    \caption{
    Online learning on various datasets.
    Cumulative accuracy in \% after having seen 2k training examples evaluated after each prediction starting with random states ({\principle}, {\metarnn}, HebbianFW, FWMemory) or random parameters (SGD).
    Standard deviations are over 32 meta test training runs.
    Meta testing is done on the official test set of each dataset.
    Meta training is on subsets of datasets excluding the Sum Sign dataset.
    Unseen tasks, most relevant from a general-purpose LA perspective, are opaque.
    }
    \label{fig:all-results}
    \vspace{-5mm}
\end{figure}

To better understand how different meta training distributions and meta test datasets affect {\approach} and our baselines, we present several different combinations in \autoref{fig:all-results}.
The opaque bars represent tasks that were not seen during meta training and are thus most relevant for this analysis.
This includes four additional datasets:
(1) \emph{Kuzushiji MNIST}~\citep{clanuwat2018deep} with 10 classes,
(2) \emph{EMNIST}~\citep{cohen2017emnist} with 62 classes,
(3) A randomly generated classification dataset (\emph{Random}) with $20$ data points that changes with each step in the outer loop,
and (4) \emph{Sum Sign} which generates random inputs and requires classifying the sign of the sum of all inputs.
Meta training is done over 500 randomly drawn samples per outer iteration.
Each algorithm is meta trained for 10k outer iterations.
Inputs are randomly projected as in \autoref{sec:robustness} (for VSML; the baselines did not benefit from these augmentations).
We again report the cumulative accuracy on all data seen since the beginning of meta test training.
We compare to SGD with a single layer, matching the architecture of {\principle}, and a hidden layer, matching the number of weights to the size of $V_L$ in {\principle}.
We also have included a Hebbian fast weight baseline~\citep{Miconi2018DifferentiableBackpropagation} and an external (fast weight) memory approach~\citep{schlag2021learning}.

We observe that {\principle} generalizes much better than {\metarnns}, Hebbian fast weights, and the external memory.
These baselines overfit to the training environments.
Notably, VSML even generalizes to the unseen tasks \emph{Random} and \emph{Sum Sign} which have no shared structure with the other datasets.
In many cases {\principle}'s performance is similar to SGD but a little more sample efficient in the beginning of training (learning curves in \autoref{sec:additional-experiments}).
This suggests that our meta learned LAs are good at quickly associating new inputs with their labels.
We further investigate this in the next \autoref{sec:introspection}.

\section{Analysis}\label{sec:introspection}

\begin{wrapfigure}{r}{0.7\textwidth}
    \vspace{-6mm}
    \centering
    \includegraphics[width=0.7\columnwidth]{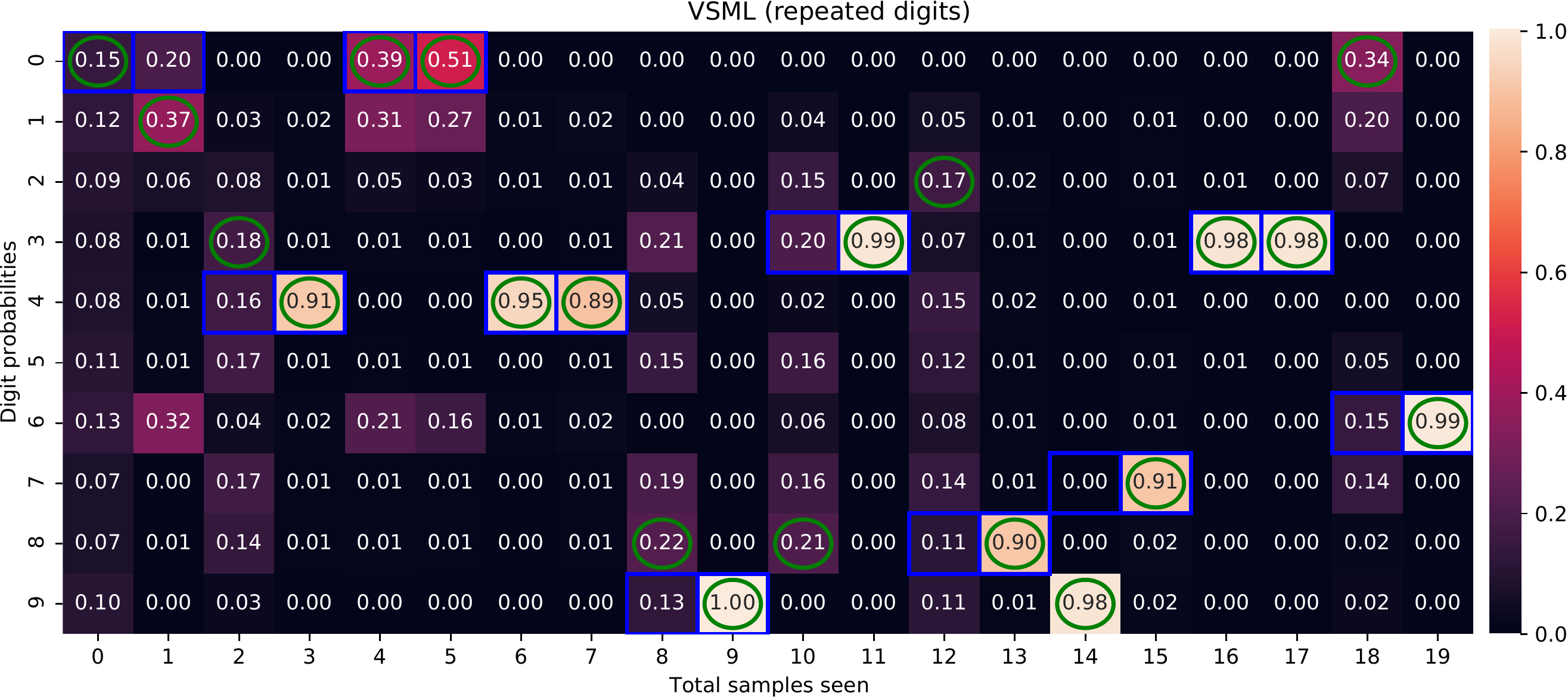}
    \includegraphics[width=0.7\columnwidth]{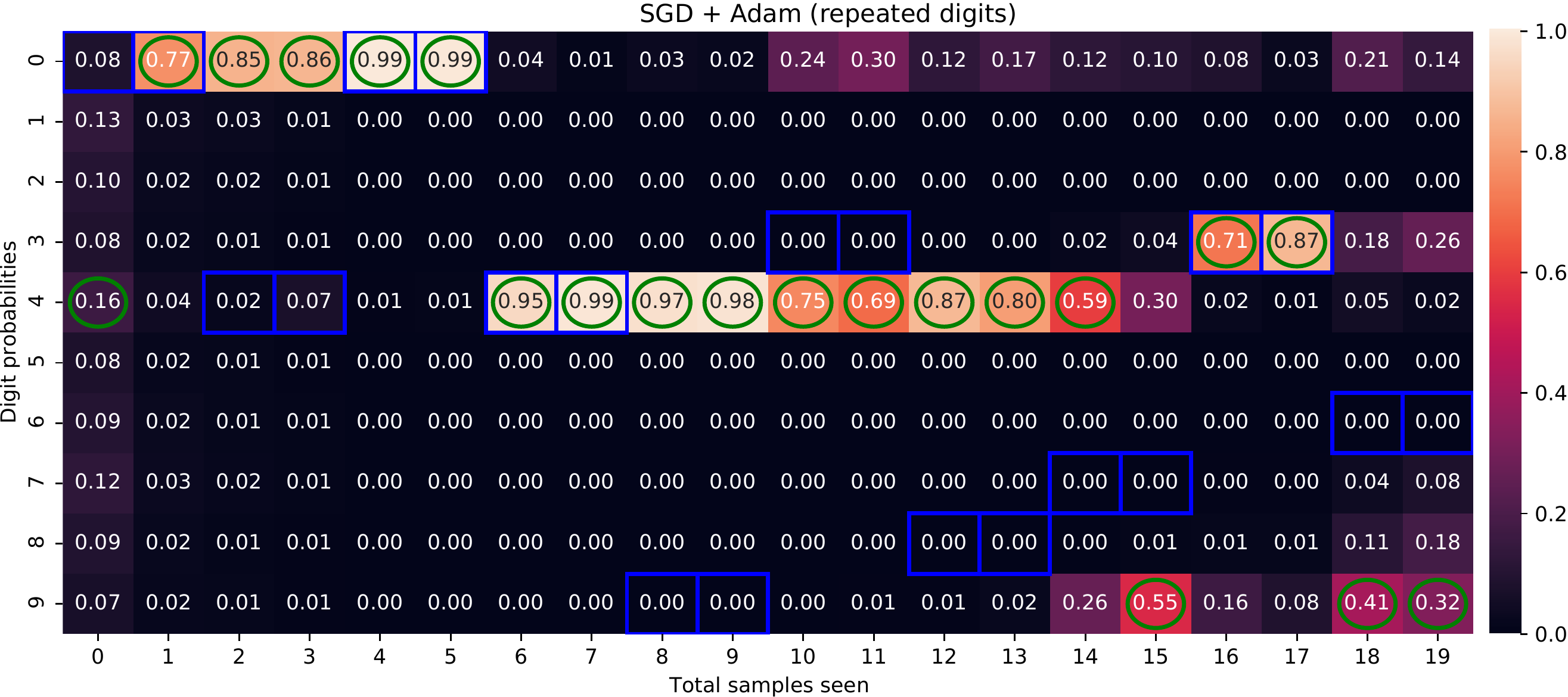}
    \caption{
        Introspection of how output probabilities change after observing an input and the prediction error when meta testing on MNIST with two examples for each type.
        We highlight the ground truth class $\color{label} \square$ as well as the predicted class $\color{prediction} \bigcirc$.
        The top plot shows {\principle} quickly associating the input images with the right label, almost always making the right prediction the second time with high confidence.
        The bottom plot shows the same dataset processed by SGD with Adam which fails to learn quickly.
    }
    \label{fig:action-introspection}
    \vspace{-5mm}
\end{wrapfigure}

Given that {\principle} seems to learn faster than online gradient descent in many cases we would like to qualitatively investigate how learning differs.
We first meta train on the full MNIST dataset as before.
During meta testing, we plot the output probabilities for each digit against the number of samples seen in \autoref{fig:action-introspection}.
We highlight the ground truth input class $\color{label} \square$ as well as the predicted class $\color{prediction} \bigcirc$.
In this case, our meta test dataset consists of MNIST digits with two examples of each type.
The same digit is always repeated twice.
This allows us to observe and visualize the effect with only a few examples.
We have done the same introspection with the full dataset in \autoref{sec:additional-experiments}.

We observe that in {\principle} almost all failed predictions are followed by the correct prediction with high certainty.
In contrast, SGD makes many incorrect predictions and fails to adapt correctly in only 20 steps.
It seems that SGD learns qualitatively different from {\principle}.
The {\sapproach} meta learns to quickly associate new inputs with their class whereas SGD fails to do so.
We tried several different SGD learning rates and considered multiple steps on the same input.
In both cases, SGD does not behave similar to VSML, either learning much slower or forgetting previous examples.
As evident from high accuracies in \autoref{fig:all-results}, {\principle} does not only memorize inputs using this strategy of fast association, but the associations generalize to future unseen inputs.

\section{Related Work}

\textbf{Memory based meta learning (Meta RNNs)}
Memory-based meta learning in RNNs~\citep{Hochreiter2001,Duan2016,Wang2016} is a simple neural meta learner (see \autoref{sec:background}).
Unfortunately, the LA encoded in the RNN parameters is largely overparameterized which leads to overfitting.
{\principle} demonstrates that weight sharing can address this issue resulting in more general-purpose LAs.

\textbf{Learned Learning Rules / Fast Weights}
NNs that generate or change the weights of another or the same NN are known as fast weight programmers~\citep{Schmidhuber:91fastweights}, hypernetworks~\citep{Ha2016}, NNs with synaptic plasticity~\citep{Miconi2018DifferentiableBackpropagation} or learned learning rules~\citep{Bengio1992OnRule} (see \autoref{sec:background}).
In {\principle} we do not require explicit architectures for weight updates as weights are emergent from RNN state updates.
In addition to the learning rule, we implicitly learn how the neural forward computation is defined.
Concurrent to this work, fast weights have also been used to meta learn more general LAs~\citep{sandler2021meta}.

\textbf{Learned gradient-based optimizers}
Meta learning has been used to find optimizers that update the parameters of a model by taking the loss and gradient with respect to these parameters as an input~\citep{Ravia,Andrychowicz2016,Li2016b,Metz2020}.
In this work, we are interested in meta learning that does not rely on fixed gradient calculation in the inner loop.

\textbf{Discrete program search}
An interesting alternative to distributed variable updates in {\principle} is meta learning via discrete program search~\citep{Schmidhuber:94kol,Real2020}.
In this paradigm, a separate programming language needs to be defined that gives rise to neural computation when its instructions are combined.
This led to the automated rediscovery of backpropagation~\citep{Real2020}.
In VSML we demonstrate that a symbolic programming language is not required and general-purpose LAs can be discovered and encoded in variable-shared RNNs.
Search over neural network parameters is usually easier compared to symbolic program search due to smoothness in the loss landscape.

\textbf{Multi-agent systems}
In the reinforcement learning setting multiple agents can be modeled with shared parameters~\citep{SimsKarl1994,Pathak2019,huang2020smp}, also in the context of meta learning~\citep{Rosa2019}.
This is related to the variable sharing in {\principle} depending on how the agent-environment boundary is drawn.
Unlike these works, we demonstrate the advantage of variable sharing in meta learning more general-purpose LAs and present a weight update interpretation.

\section{Discussion and Limitations}

The research community has perfected the art of leveraging backpropagation for learning for a long time.
At the same time, there are open questions such as how to minimize memory requirements, effectively learn online and continually, learn sample efficiently, learn without separate backward phases, and others.
{\principle} suggests that instead of building on top of backpropagation as a fixed routine, meta learning offers an alternative to discover general-purpose LAs.
Nevertheless, this paper is only a proof of concept---until now we have only investigated small-scale problems and performance does not yet quite match the mini-batched setting with large quantities of data.
In particular, we observed premature convergence of the solution at meta test time which calls for further investigations.
Scaling our system to harder problems and larger meta task distributions will be important future work.

The computational cost of the current {\principle} variant is also larger than the one of standard backpropagation.
If we run a sub-RNN for each weight in a standard NN with $W$ weights, the cost is in $O(WN^2)$, where $N$ is the state size of a sub-RNN.
If $N$ is small enough, and our experiments suggest small $N$ may be feasible, this may be an acceptable cost.
However, {\principle} is not bound to the interpretation of a sub-RNN as one weight.
Future work may relax this particular choice.

Meta optimization is also prone to local minima.
In particular, when the number of ticks between input and feedback increases (e.g. deeper architectures), credit assignment becomes harder.
Early experiments suggest that diverse meta task distributions can help mitigate these issues.
Additionally, learning horizons are limited when using backprop-based meta optimization.
Using ES allowed for training across longer horizons and more stable optimization.

{\principle} can also be viewed as regularizing the NN weights that encode the LA through a representational bottleneck.
It is conceivable that LA generalization as obtained by {\principle} can also be achieved through other regularization techniques.
Unlike most regularizers, {\principle} also introduces substantial reuse of the same learning principle and permutation invariance through variable sharing.

\section{Conclusion}

We introduced {\principlefull} (VSML), a simple principle of weight sharing and sparsity for meta learning powerful learning algorithms (LAs).
Our implementation replaces the weights of a neural network with tiny LSTMs that share parameters.
We discuss connections to meta recurrent neural networks, fast weight generators (hyper networks), and learned learning rules.

Using \emph{learning algorithm cloning}, {\approachfull} can learn to implement the backpropagation algorithm and its parameter updates encoded implicitly in the recurrent dynamics.
On MNIST it learns to predict well without any human-designed explicit computational graph for gradient calculation.

{\principle} can meta learn from scratch supervised LAs that do not explicitly rely on gradient computation and that \emph{generalize to unseen datasets}.
Introspection reveals that {\principle} LAs learn by fast association in a way that is qualitatively different from stochastic gradient descent.
This leads to gains in sample efficiency.
Future work will focus on reinforcement learning settings, improvements of meta learning, larger task distributions, and learning over longer horizons.

\newpage

\section*{Acknowledgements}
We thank Sjoerd van Steenkiste, Imanol Schlag, Kazuki Irie, and the anonymous reviewers for their comments and feedback.
This work was supported by the ERC Advanced Grant (no: 742870) and computational resources by the Swiss National Supercomputing Centre (CSCS, projects s978 and s1041).
We also thank NVIDIA Corporation for donating several DGX machines as part of the Pioneers of AI Research Award, IBM for donating a Minsky machine, and weights \& biases~\citep{biewald2020} for their great experiment tracking software and support.

\bibliography{refs}

\newpage

\newpage
\appendix

\section{Derivations}\label{sec:rnn-equiv-derivation}

\begin{theorem}
The weight matrices $W$ and $C$ used to compute {\approach} from \autoref{eq:interact} can be expressed as a standard RNN with weight matrix $\tilde W$ (\autoref{eq:sparse-shared-equiv}) such that
\begin{align}\label{eq:rnn-equiv}
    s_{abj} &\leftarrow \sigma(\sum_i s_{abi} W_{ij} + \sum_{c,i} s_{cai} C_{ij}) \\
    &= \sigma(\sum_{c,d,i} s_{cdi} \tilde W_{cdiabj}).
\end{align}
The weight matrix $\tilde W$ has entries of zero and shared entries given by \autoref{eq:sparse-shared}.
\begin{equation*}
    \tilde W_{cdiabj} = \begin{cases}
        C_{ij}, & \text{if } d = a \land ( d \neq b \lor c \neq a ). \\
        W_{ij}, & \text{if } d \neq a \land d = b \land c = a. \\
        C_{ij} + W_{ij}, & \text{if } d = a \land d = b \land c = a. \\
        0, & \text{otherwise.}
    \end{cases}
    \tag{\ref{eq:sparse-shared} revisited}
\end{equation*}
\end{theorem}

\begin{proof}
We rearrange $\tilde W$ into two separate weight matrices
\begin{align}
    & \sum_{c,d,i} s_{cdi} \tilde W_{cdiabj} \\
    = & \sum_{c,d,i} s_{cdi} A_{cdiabj} + \sum_{c,d,i} s_{cdi} (\tilde W - A)_{cdiabj}. \label{eq:separated_tildeW}
\end{align}
Then assuming $A_{cdiabj} = (d \equiv b)(c \equiv a) W_{ij}$, where $x \equiv y$ equals $1$ iff $x$ and $y$ are equal and $0$ otherwise, it holds that
\begin{equation}
    \sum_{c,d,i} s_{cdi} A_{cdiabj} = \sum_i s_{abi} W_{ij}.
\end{equation}
Further, assuming $(\tilde W - A)_{cdiabj} = (d \equiv a) C_{ij}$ we obtain
\begin{equation}
    \sum_{c,d,i} s_{cdi} (\tilde W - A)_{cdiabj} = \sum_{c,i} s_{cai} C_{ij}.
\end{equation}
Finally, solving both conditions for $\tilde W$ gives
\begin{equation}
    \tilde W_{cdiabj} = (d \equiv a) C_{ij} + (d \equiv b)(c \equiv a)W_{ij},
\end{equation}
which we rewrite in tabular notation:
\begin{equation}
    \tilde W_{cdiabj} = \begin{cases}
        C_{ij}, & \text{if } d = a \land ( d \neq b \lor c \neq a ). \\
        W_{ij}, & \text{if } d \neq a \land d = b \land c = a. \\
        C_{ij} + W_{ij}, & \text{if } d = a \land d = b \land c = a. \\
        0, & \text{otherwise.}
    \end{cases}
\end{equation}
Thus, \autoref{eq:rnn-equiv} holds and any weight matrices $W$ and $C$ can be expressed by a single weight matrix $\tilde W$.
\end{proof}

\newpage
\section{Additional Experiments}\label{sec:additional-experiments}

\subsection{Learning algorithm cloning}

\begin{wrapfigure}{r}{0.5\textwidth}
    \centering
    \includegraphics[width=0.5\columnwidth]{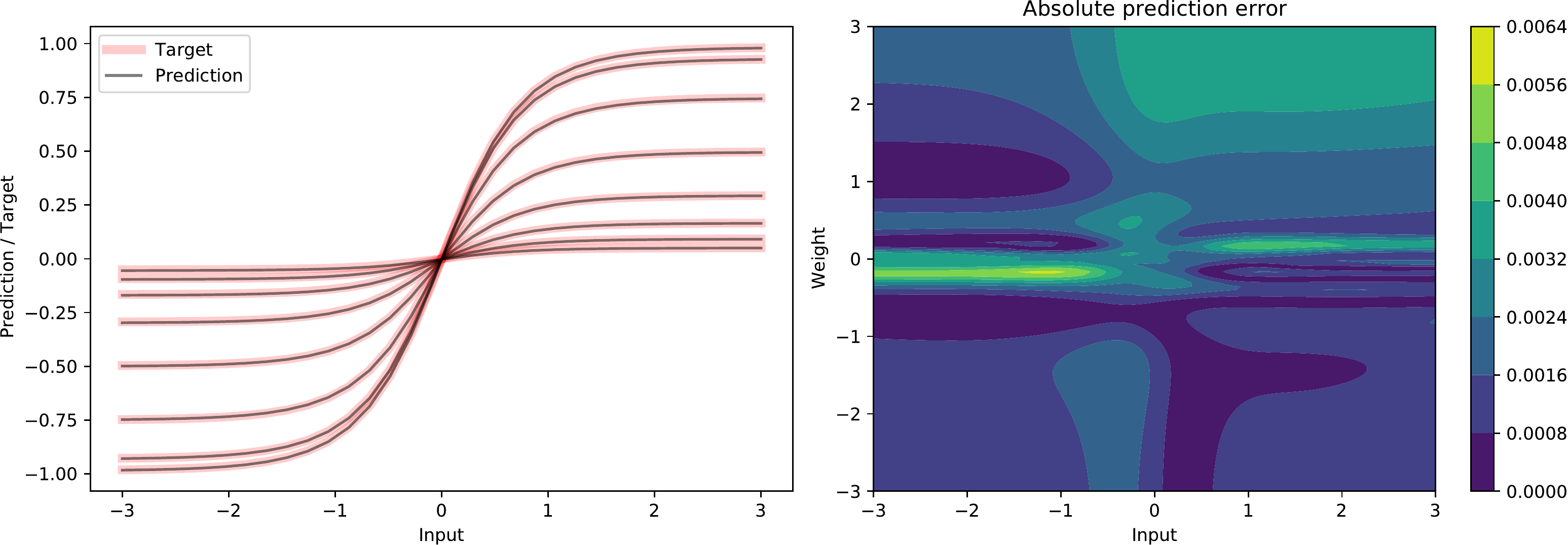}
    \caption{We are optimizing {\approach} to implement neural forward computation such that for different inputs and weights a tanh-activated multiplicative interaction is produced (left), with different lines for different $w$.
    These neural dynamics are not exactly matched everywhere (right), but the error is relatively small.
    }
    \label{fig:neural_tanh}
    \vspace{-5mm}
\end{wrapfigure}

\paragraph{{\approach} can implement neural forward computation}\label{sec:lstm_neural}
In this experiment, we optimize the {\sapproach} to compute $y = tanh(x)w$.
\autoref{fig:neural_tanh} (left) shows how for different inputs $x$ and weights $w$ the LSTM produces the correct target value, including the multiplicative interaction.
The heat-map (right) shows that low prediction errors are produced but the target dynamics are not perfectly matched.
We repeat these LSTMs in line with \autoref{eq:message-forward-backward} to obtain an `emergent' neural network.

\paragraph{Learning Algorithm Cloning Curriculum}
In principle, backpropagation can be simply cloned on random data such that forward computation implements multiplicative activation-weight interaction and backward computation passes an error signal back given previous forward activations.
If the previous forward activations are fed as an input one could stack {\approach} that implement these two operations to mimic arbitrarily deep NNs.
By purely training on random data and unrolling for one step, we can successfully learn on MNIST and Fashion MNIST in the shallow setting.
For deeper models, in practice, cloning errors accumulate and input and state distributions shift.
To achieve learning in deeper networks we have used a curriculum on random and MNIST data.
We first match the forward activations, backward errors, and weight updates for a shallow network.
Next, we use a deep network and provide intermediate errors by a ground truth network.
Finally, we remove intermediate errors and use the RNN's intermediate predictions that are now close to the ground truth.
The final {\sapproach} can be used to train both shallow (\autoref{fig:impl_backprop_shallow}) and deep configurations (\autoref{fig:impl_backprop_deep}).

\begin{figure}[t]
    \centering
    \hfill
    \begin{subfigure}[b]{0.49\textwidth}
    \includegraphics[width=\textwidth]{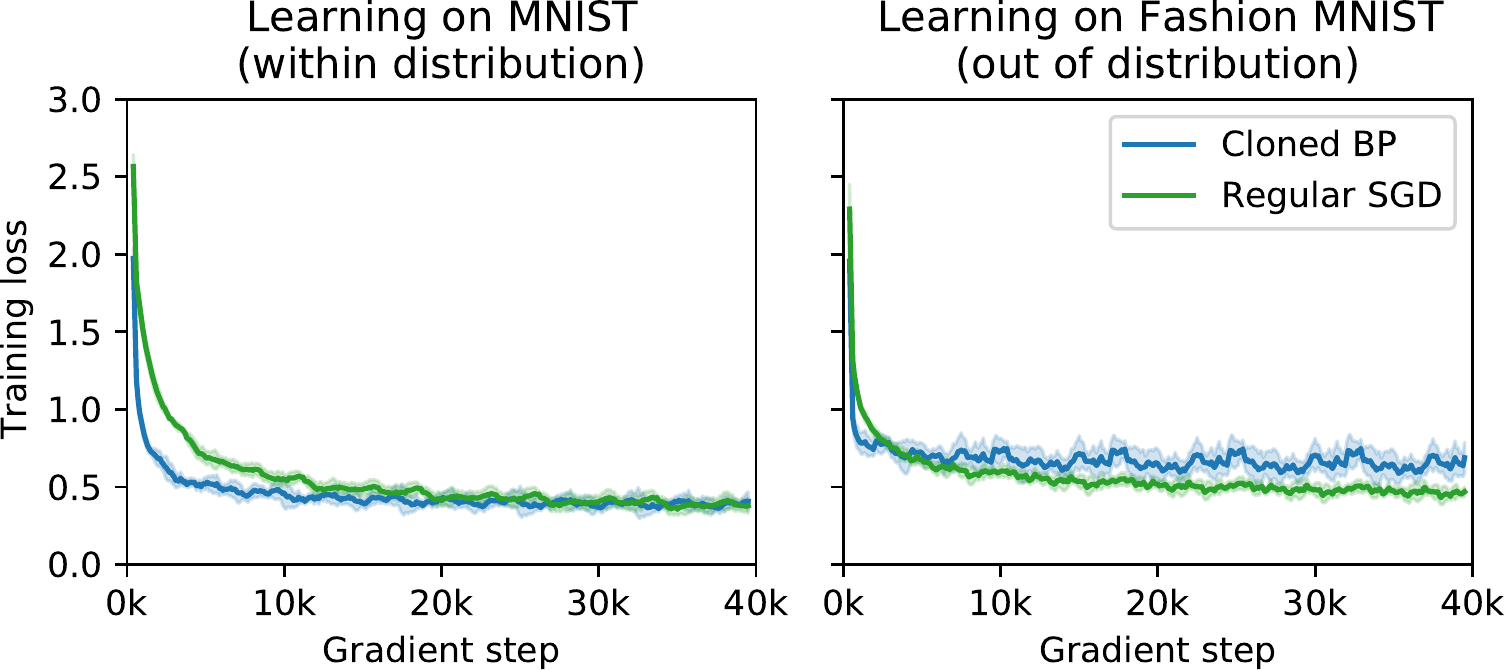}
    \includegraphics[width=\textwidth]{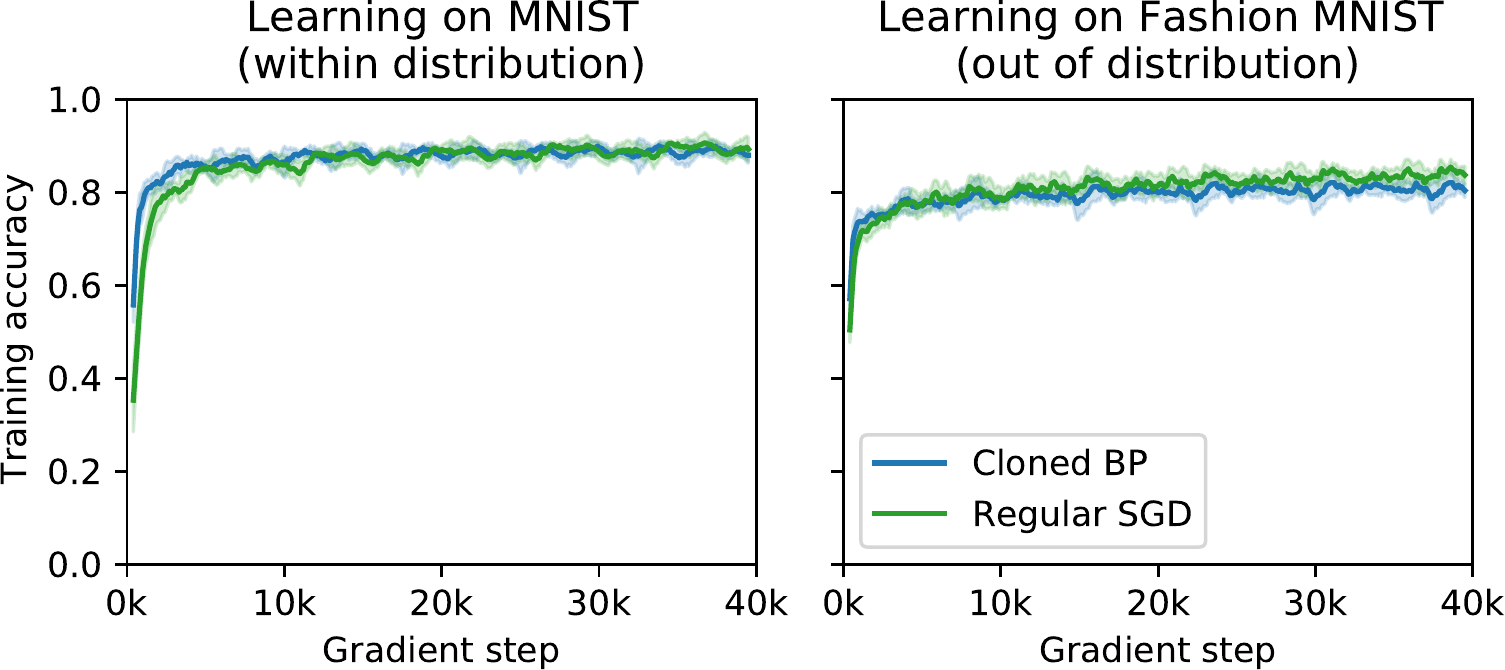}
    \caption{Shallow network arrangement.}
    \label{fig:impl_backprop_shallow}
    \end{subfigure}
    \hfill
    \begin{subfigure}[b]{0.49\textwidth}
    \includegraphics[width=\textwidth]{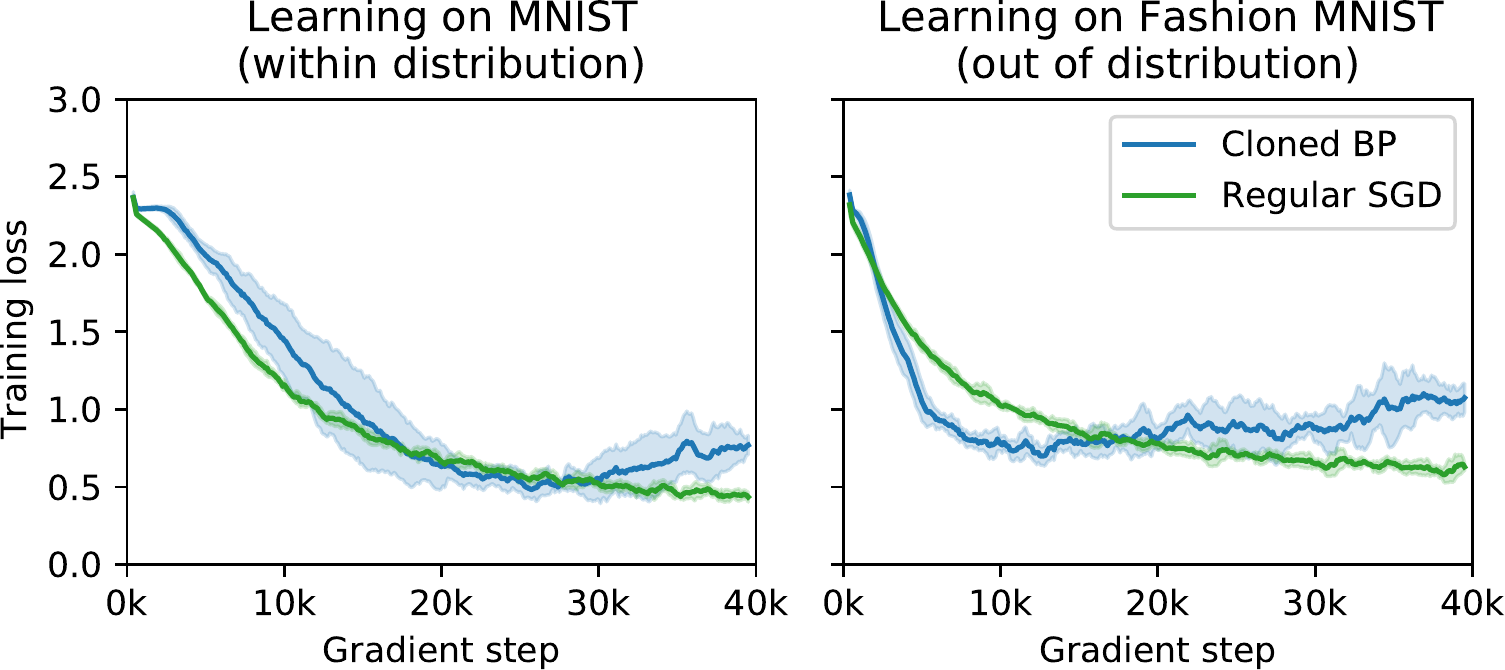}
    \includegraphics[width=\textwidth]{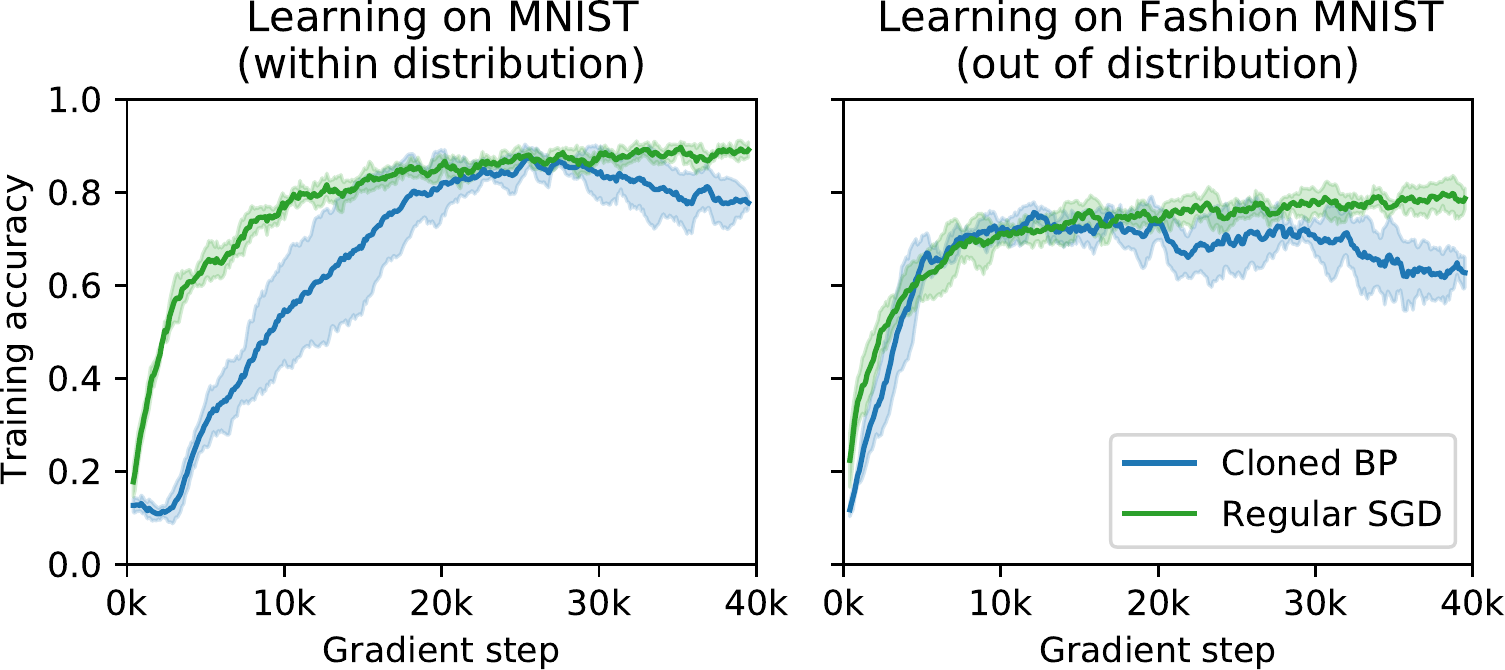}
    \caption{Deep (1 hidden layer) network arrangement.}
    \label{fig:impl_backprop_deep}
    \end{subfigure}
    \caption{Additional experiments with {\approach} implementing backpropagation.
    Standard deviations are over 6 seeds.}
\end{figure}

\subsection{Meta learning from scratch}

\paragraph{Meta testing learning curves \& sample efficiency}
In \autoref{fig:all-results} we only showed accuracies after 2k steps.
\autoref{fig:learning-curves} provides the entire meta test training trajectories for a subset of all configurations.
Furthermore, in \autoref{fig:sample-efficiency} we show the cumulative accuracy on the first 100 examples.
From both figures, it is evident that learning at the beginning is accelerated compared to SGD with Adam.
Also compare with our introspection from \autoref{sec:introspection}.

\begin{figure}
    \centering
    \includegraphics[width=\textwidth]{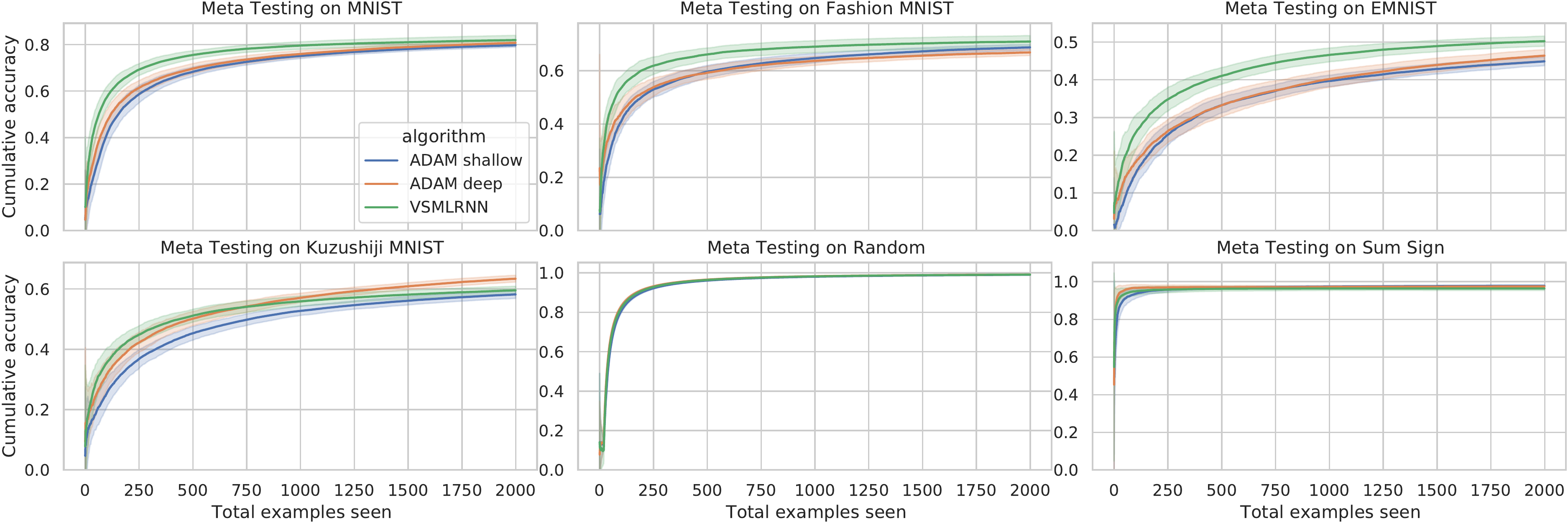}
    \caption{Meta testing learning curves. All 6 meta test tasks are unseen. {\sapproach} has been meta trained on MNIST, Fashion MNIST, EMNIST, KMNIST, and Random, excluding the respective dataset that is being meta tested on. Standard deviations are over 32 seeds.
    }
    \label{fig:learning-curves}
\end{figure}

\begin{figure}
    \centering
    \includegraphics[width=\textwidth]{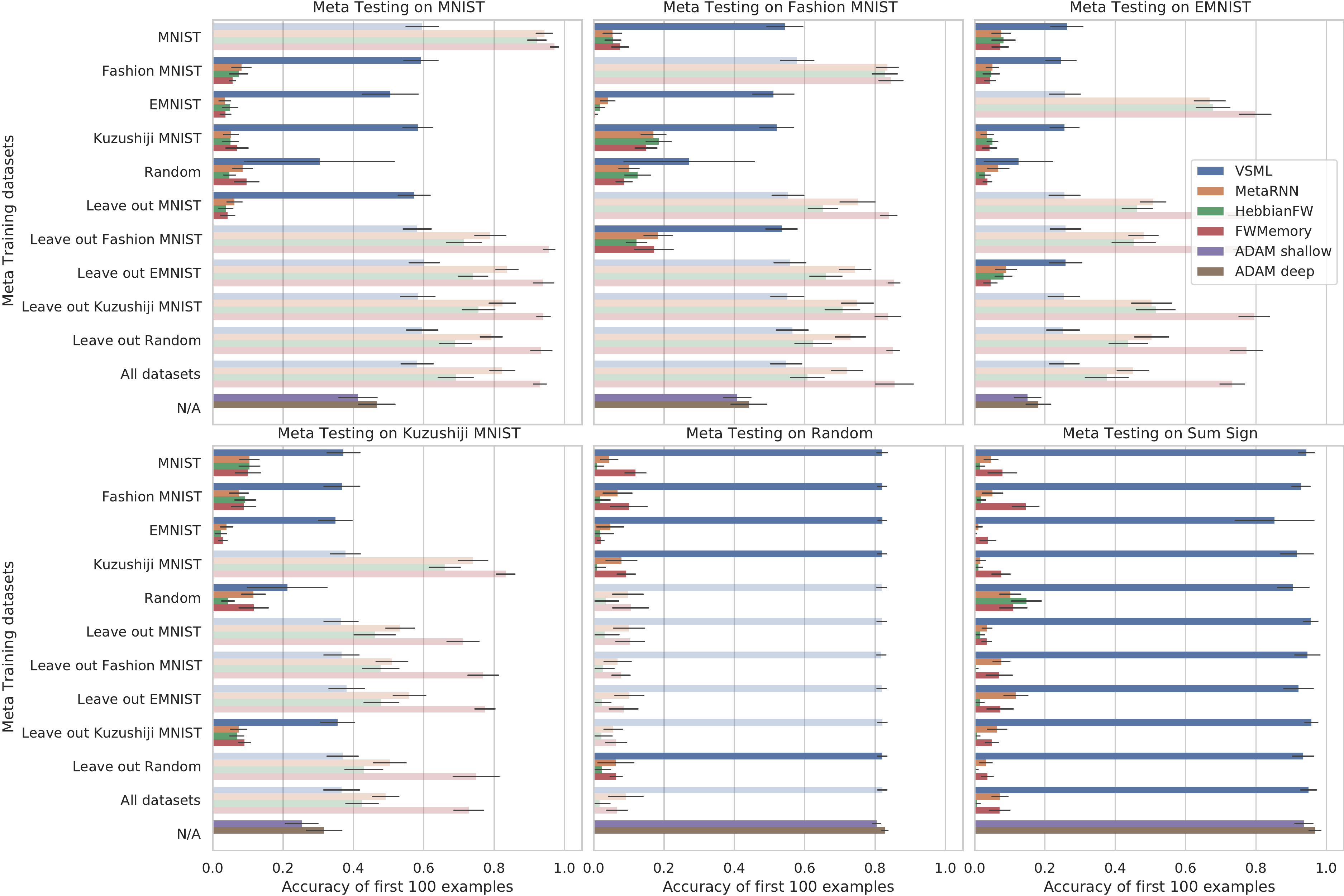}
    \caption{
    Online learning on various datasets.
    Cumulative accuracy in \% after having seen \textbf{100 training examples} evaluated after each prediction starting with random states ({\principle}, {\metarnn}, HebbianFW, FWMemory) or random parameters (SGD).
    Standard deviations are over 32 meta test training runs.
    Meta testing is done on the official test set of each dataset.
    Meta training is on subsets of datasets excluding the Sum Sign dataset.
    Unseen tasks, most relevant from a general-purpose LA perspective, are opaque.
    }
    \label{fig:sample-efficiency}
\end{figure}

\paragraph{Ablation: Projection augmentations}

\begin{figure}
    \centering
    \includegraphics[width=\textwidth]{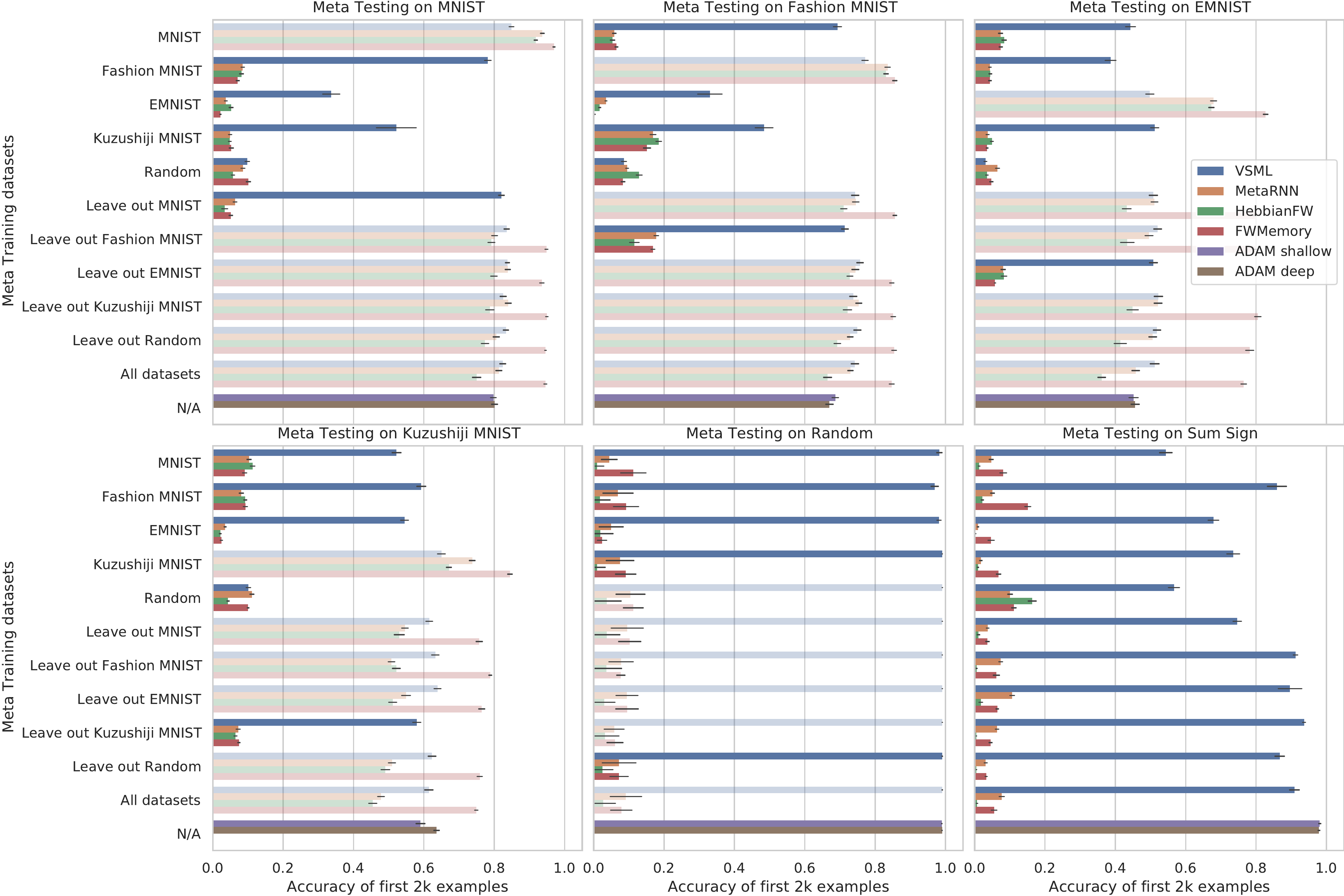}
    \caption{
    Same as figure \autoref{fig:all-results} and \autoref{fig:sample-efficiency} but with accuracies after having seen 2k training examples and \textbf{no random projections for all methods} during meta training.
    }
    \label{fig:not-proj-inputs}
\end{figure}
\begin{figure}
    \centering
    \includegraphics[width=\textwidth]{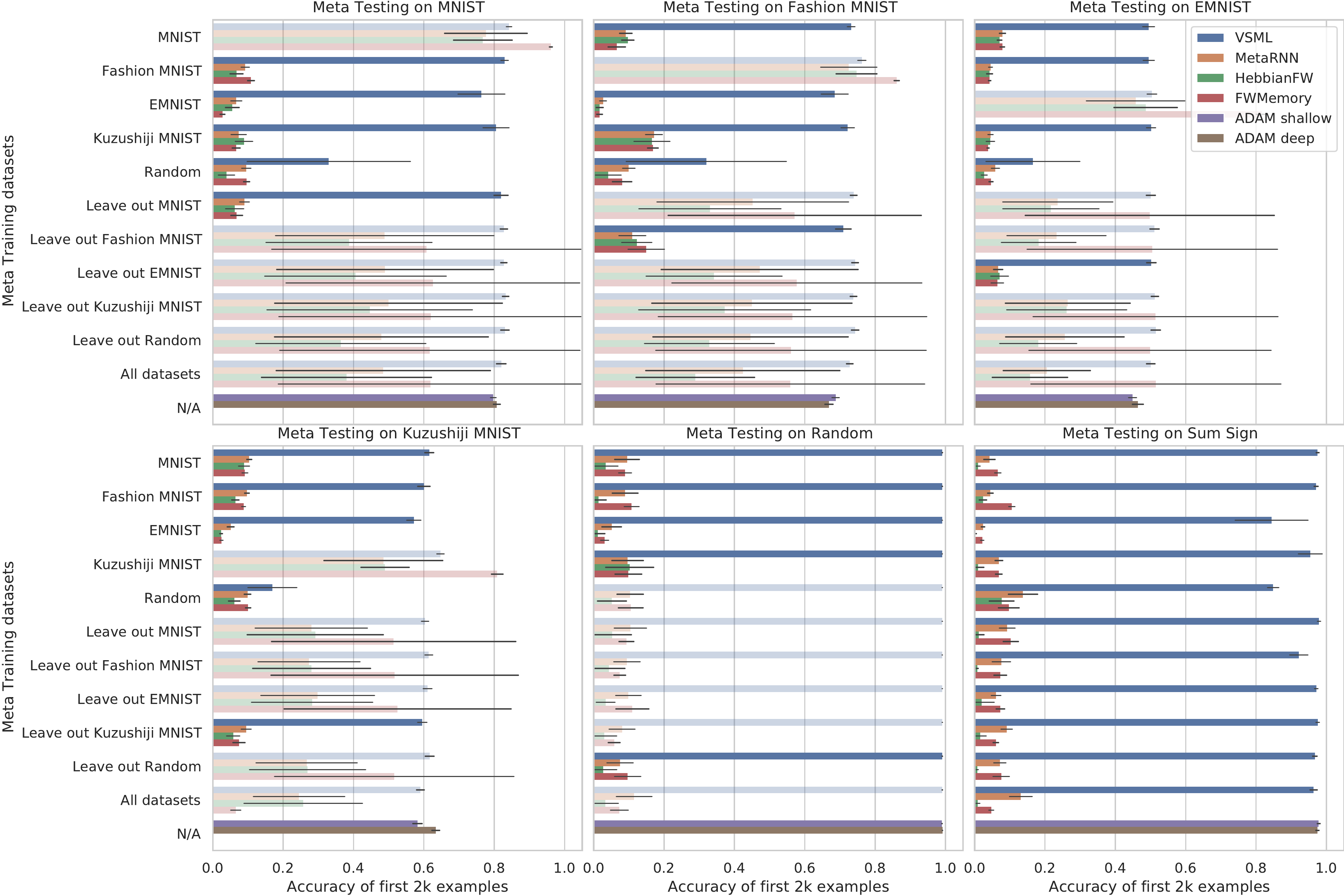}
    \caption{
    Same as figure \autoref{fig:all-results} and \autoref{fig:sample-efficiency} but with accuracies after having seen 2k training examples and \textbf{random projections for all methods including baselines} during meta training.
    }
    \label{fig:proj-inputs}
\end{figure}
In the main text (\autoref{fig:all-results}) we have randomly projected inputs during VSML meta training.
When not randomly projecting inputs (\autoref{fig:not-proj-inputs}), generalization of {\principle} is slightly reduced.
In \autoref{fig:proj-inputs} we have enabled these augmentations for all methods, including the baselines.
While {\principle} benefits from the augmentations, the {\metarnn}, Hebbian fast weights, and external memory baselines do not increase their generalization significantly with those enabled.
In \autoref{fig:mnist-augmentations} we show meta test training curves for both the augmented as well as non-augmented case.

\begin{wrapfigure}{r}{0.5\textwidth}
    \vspace{-5mm}
    \centering
    \includegraphics[width=0.5\columnwidth]{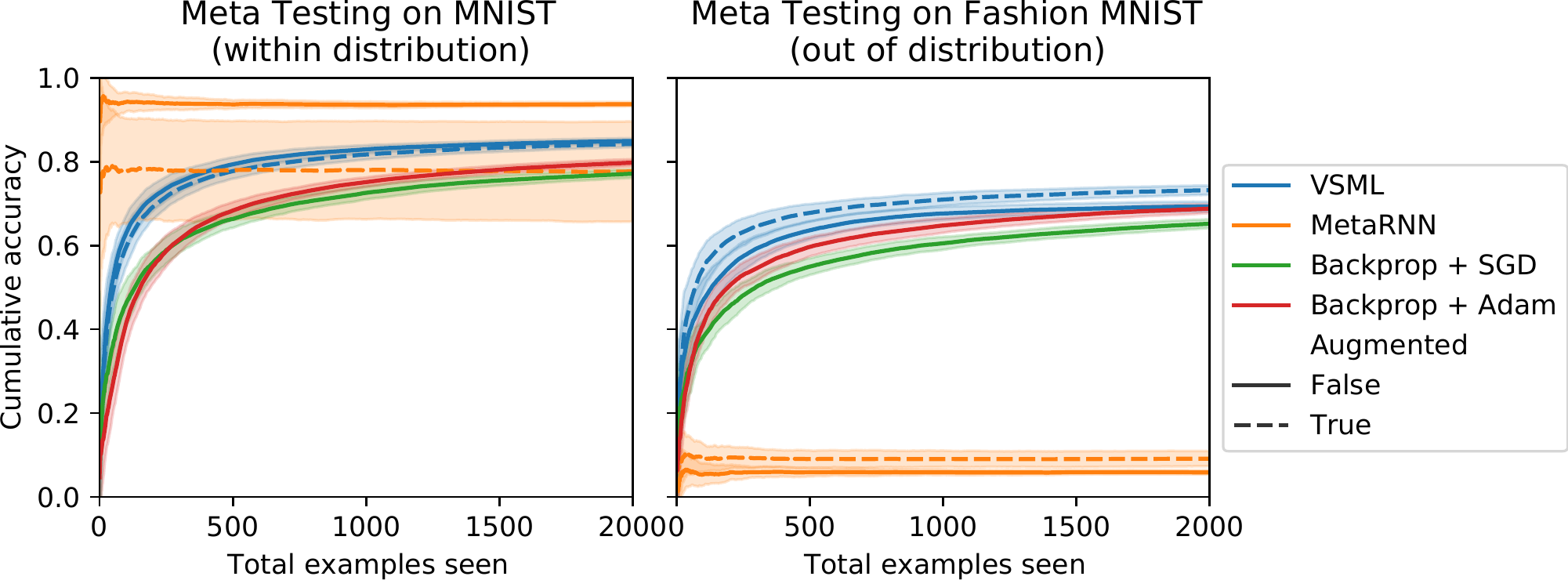}
    \caption{
    On the MNIST meta training example from \autoref{fig:scratch} we plot the effect of adding the random projection augmentation to VSML and the Meta RNN.
    The Fashion MNIST performance (out of distribution) is slightly improved for VSML while the effect on the Meta RNN is limited.
    }
    \label{fig:mnist-augmentations}
    \vspace{-5mm}
\end{wrapfigure}

\paragraph{Introspect longer meta test training run}

\begin{figure}
    \centering
    \includegraphics[width=\textwidth]{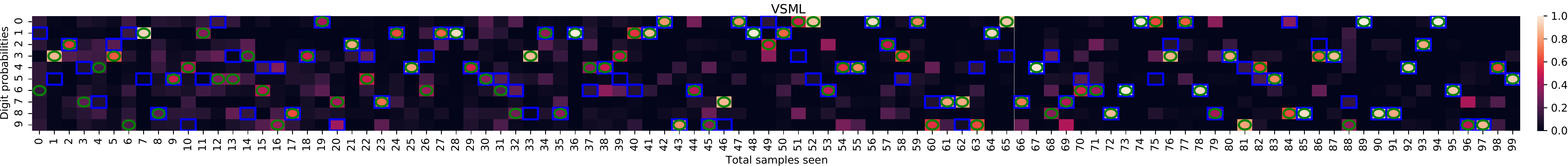}
    \includegraphics[width=\textwidth]{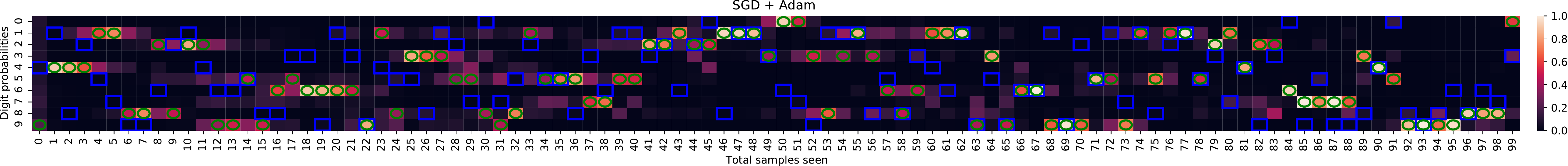}
    \caption{
        Introspection of how output probabilities change after observing an input and its error at the output units when meta testing \textbf{on the full MNIST dataset}.
        We highlight the input class $\color{label} \square$ as well as the predicted class $\color{prediction} \bigcirc$ for 100 examples in sequence.
        The top plot shows the {\sapproach} quickly associating the input images with the right label, generalizing to future inputs.
        The bottom plot shows the same dataset processed by SGD with Adam which learns significantly slower by following the gradient.
    }
    \label{fig:action-introspection-extended}
\end{figure}

Similar to \autoref{fig:action-introspection} we introspect how {\approach} learned to learn after meta training on the MNIST dataset.
In this case, we meta test for 100 steps by sampling from the full MNIST dataset in \autoref{fig:action-introspection-extended} without repeating digits.
Compared to the previous setup, learning is slower as there is a larger variety of possible inputs.
Nevertheless, we observe that {\approach} still associate inputs with their label rather quickly compared to SGD.

\paragraph{Omniglot}

In this paper, we have focused on the objective of meta learning a general-purpose learning algorithm.
Different from most contemporary meta learning approaches we tested the discovered learning algorithm on significantly different datasets to assess its generalization capabilities.
These generalization capabilities may affect the performance on standard few-shot benchmarks such as Omniglot.
In this section, we assess how VSML performs on those datasets where the tasks at meta test time are similar to those during meta training.

\begin{wrapfigure}{r}{0.5\textwidth}
    \centering
    \includegraphics[width=0.49\textwidth]{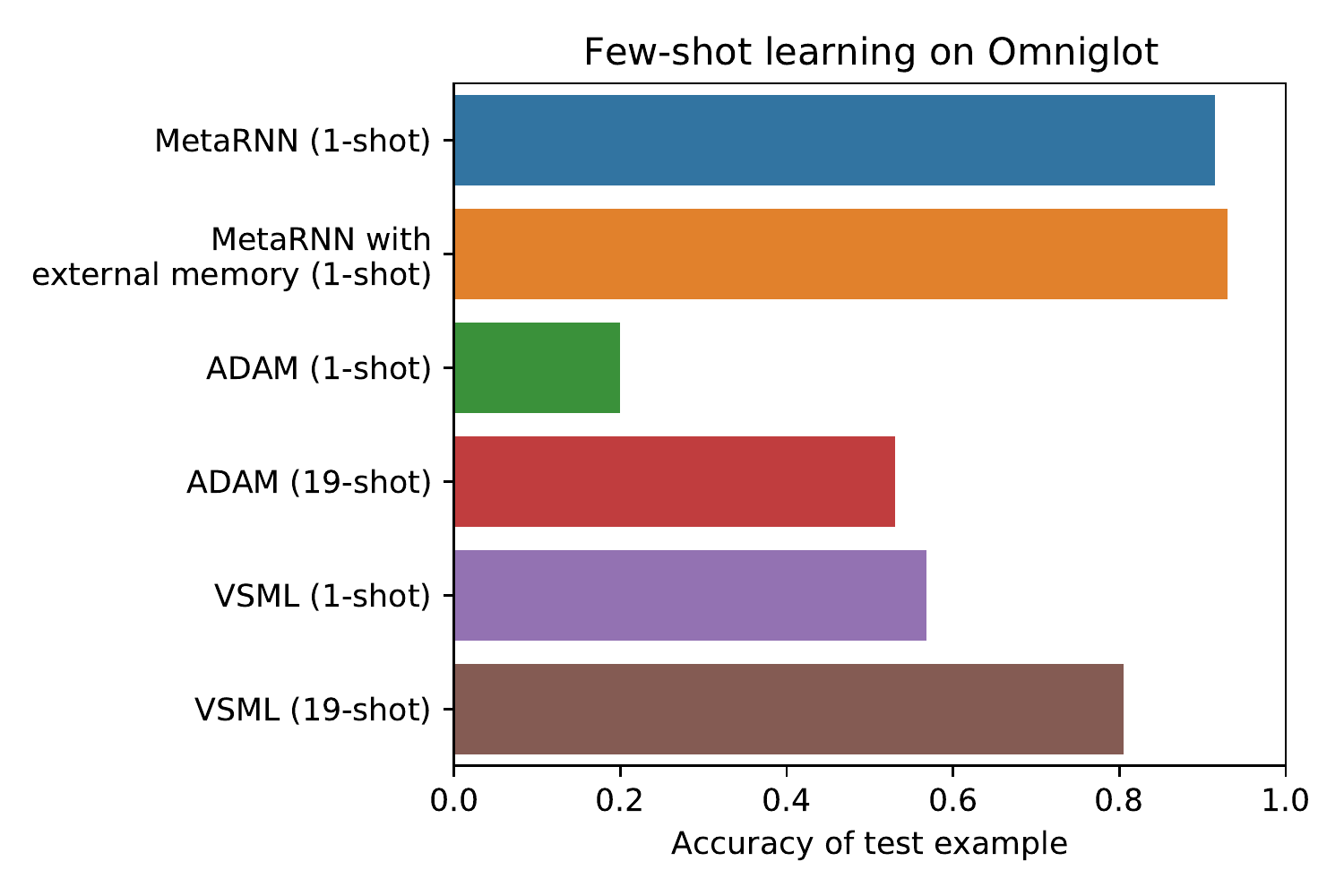}
    \caption{VSML on the Omniglot dataset.}
    \label{fig:omniglot}
\end{wrapfigure}

On Omniglot, our experimental setting corresponds to the common 5-way, 1-shot setting~\citep{Miconi2018DifferentiableBackpropagation}:
In each episode, we select 5 random classes, sample 1 instance each, and show it to the network with the label and prediction error.
Then, we sample a new random test instance from one of the 5 classes and meta train to minimize the cross-entropy on that example.
At meta test time we use unseen alphabets (classes) from the test set and report the accuracy of the test instance across 100 episodes.

The results (\autoref{fig:omniglot}) nicely demonstrate how common baselines such as the Meta RNN~\citep{Hochreiter2001,Duan2016,Wang2016} or a Meta RNN with external memory~\citep{schlag2021learning} work well in an Omniglot setting, yet fail when the gap increases between meta train and meta test, thus requiring stronger generalization (\autoref{fig:scratch}, \autoref{fig:all-results}).
In contrast, VSML generalizes well to unseen datasets, e.g. Fashion MNIST, although it does learn more slowly on Omniglot.
Finally, these new results demonstrate how VSML learns significantly faster on Omniglot compared to SGD with Adam, thus highlighting the benefits of the meta learning approach adopted in this work.

\paragraph{Short horizon bias}\label{sec:short-horizon}
\begin{wrapfigure}{r}{0.5\textwidth}
    \centering
    \includegraphics[width=0.49\textwidth]{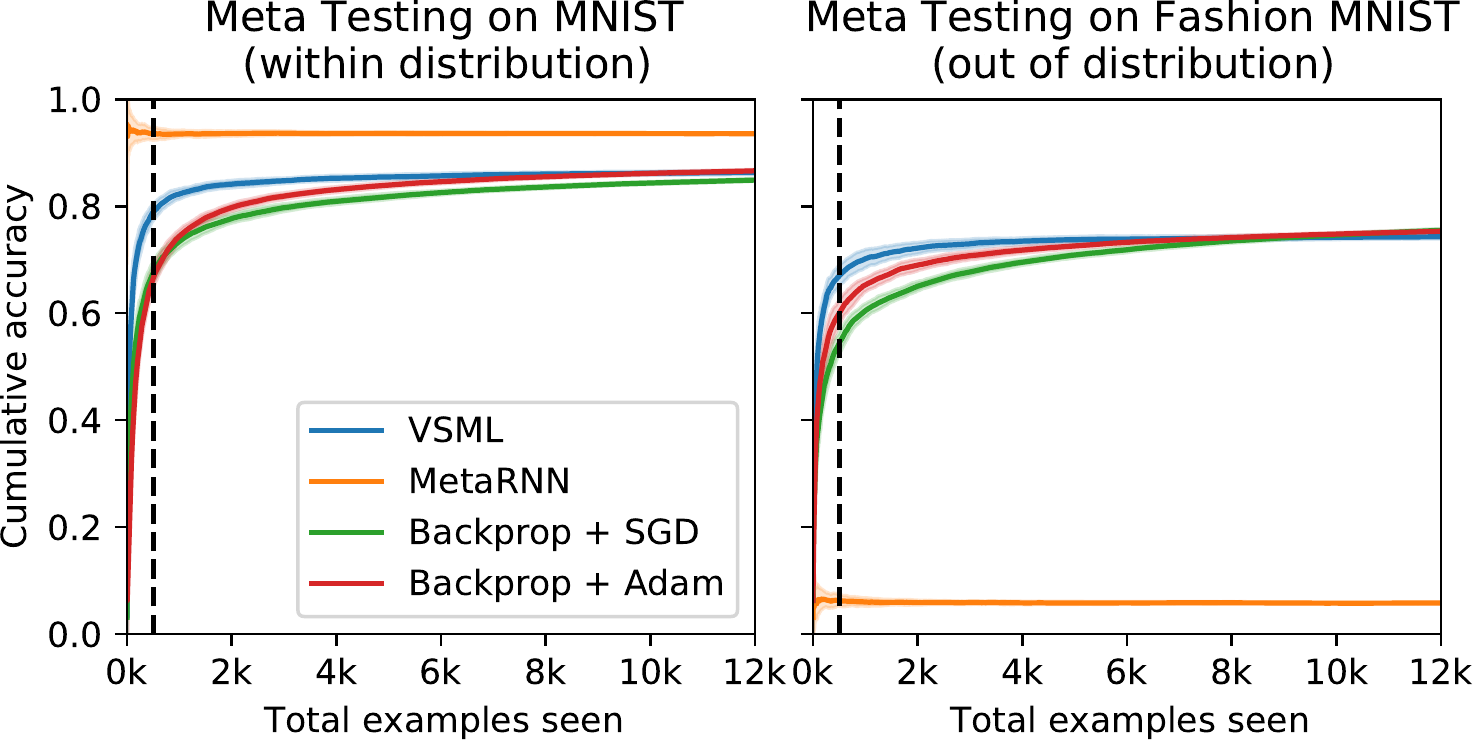}
    \caption{Short horizon bias.}
    \label{fig:short-horizon-bias}
\end{wrapfigure}
In this paper, we have observed that VSML can be significantly more sample efficient compared to backpropagation with gradient descent, in particular for the first few examples.
The longer we unroll the {\approach}, the smaller this gap becomes.
In \autoref{fig:short-horizon-bias} we run VSML for $12,000$ examples ($24,000$ RNN ticks).
From this plot, it is evident that at some point gradient descent overtakes VSML in terms of learning progress.
We call this phenomenon the \emph{short horizon bias}, where meta test training is fast in the beginning but flattens out at some horizon.
In the current version of VSML we only meta optimize the RNN for 500 examples (marked by the vertical dashed line) starting with a random initialization, not explicitly optimizing learning beyond that point, resulting in this bias.
In future work, we will investigate methods to circumvent this bias, for example by resuming from previous states (learning progress) similar to a persistent population in previous meta learning work~\citep{kirsch2020improving}.

\paragraph{Convolutional Neural Networks}\label{sec:cnns}
\begin{figure}
    \centering
    \includegraphics[width=\textwidth]{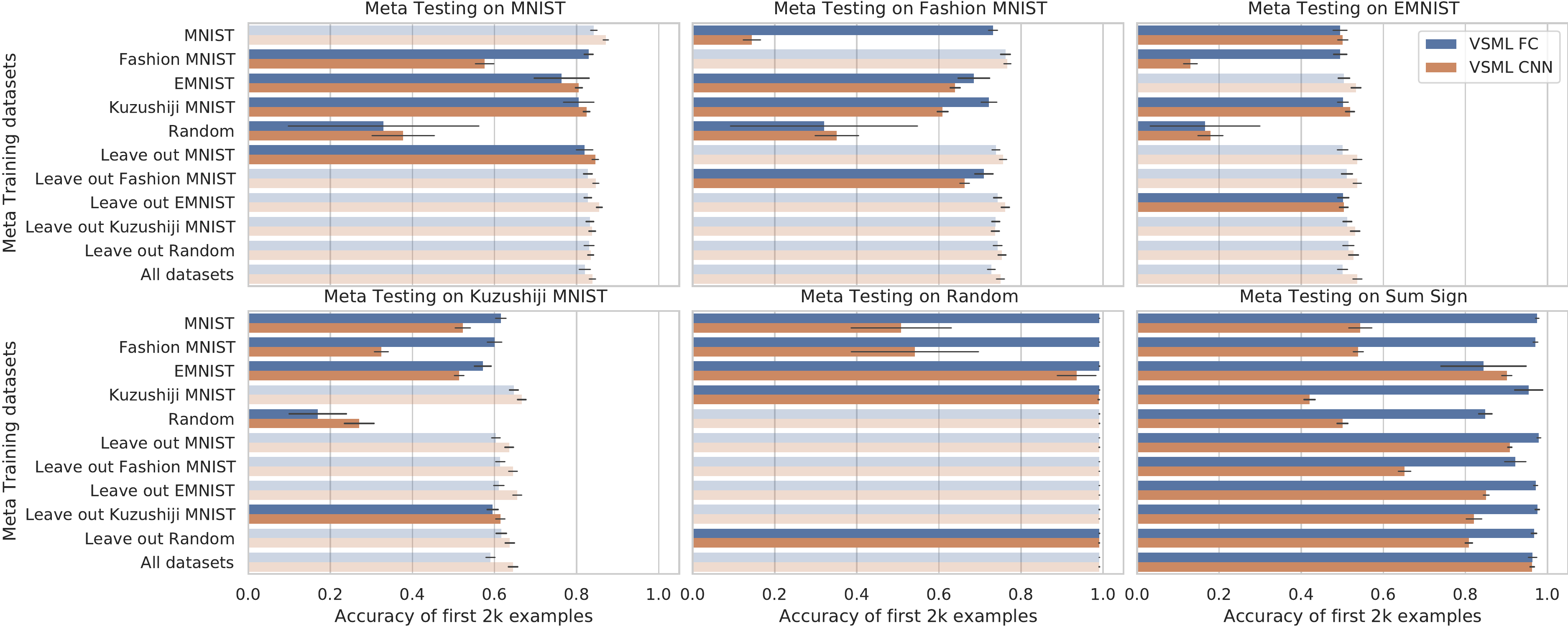}
    \caption{
    Convolutions are competitive to the standard fully connected setup.
    }
    \label{fig:compare-cnn}
\end{figure}
VSML's sub-RNNs can not only be arranged to fully connected layers, but also convolutions.
For this experiment, we have implemented a convolutional neural network (CNN) version of VSML.
This is done by replacing each weight in the kernel with a multi-dimensional RNN state and replacing the kernel multiplications with VSML sub-RNNs.
We use a convolutional layer with kernel size 3, stride 2, and 8 channels, followed by a dense layer.
On our existing datasets, it performs similar to the fully connected architecture, as can be seen in \autoref{fig:compare-cnn}.

\begin{wrapfigure}{r}{0.5\textwidth}
    \centering
    \includegraphics[width=0.49\textwidth]{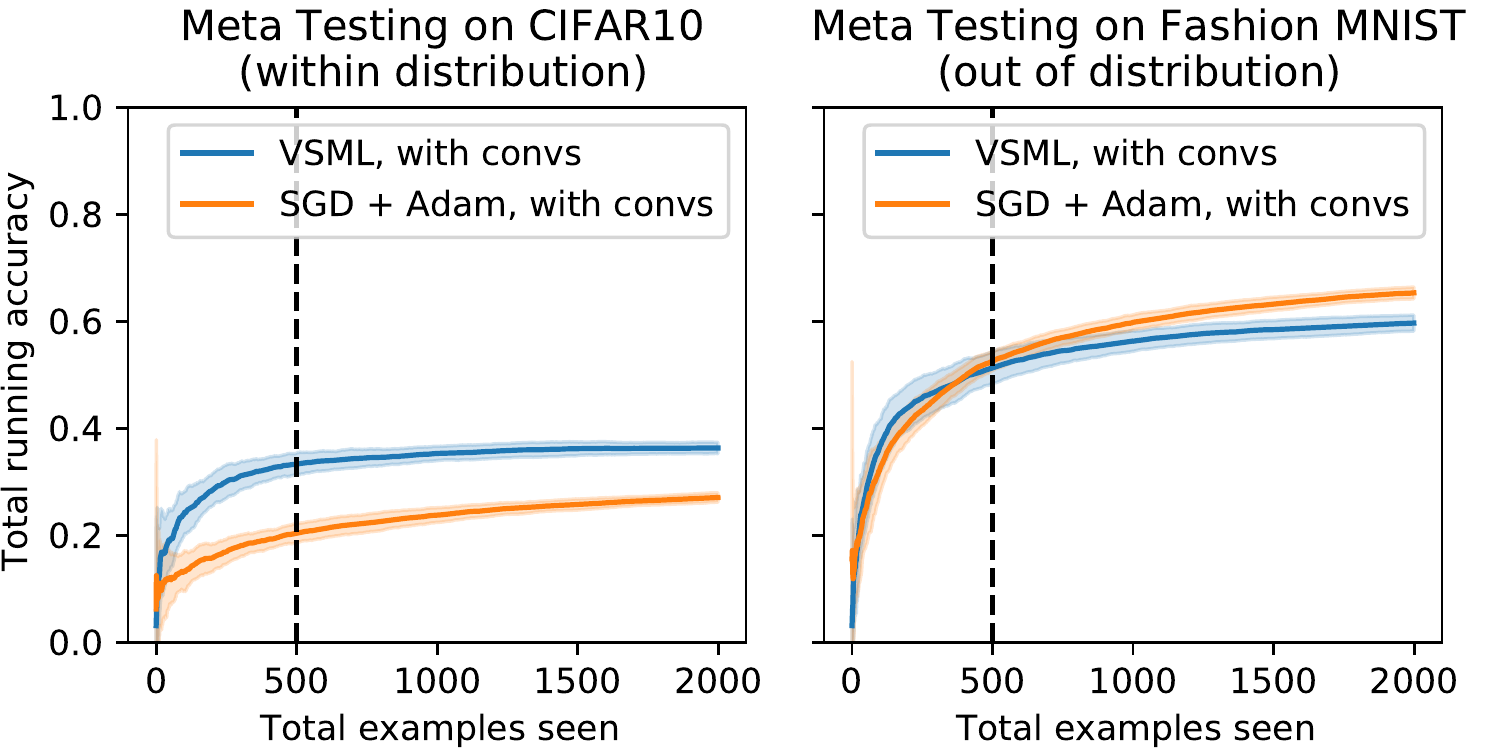}
    \caption{Meta Training on CIFAR10 with a CNN version of VSML.}
    \label{fig:cifar10}
\end{wrapfigure}
We also applied our CNN variant to CIFAR10.
Note that in this paper we are interested in the online learning setting (similar to the one of Meta RNNs).
This is a challenging problem on which gradient descent with back-propagation also struggles.
Many consecutive examples ($> 10^5$ steps) are required for learning.
Online performance is generally lower than in the batched setting which we do not explore here.
When meta training on CIFAR10 (\autoref{fig:cifar10}) we observe that meta test-time learning on CIFAR is initially faster compared to SGD while still generalizing to Fashion MNIST.
On the other hand, with a sufficiently large meta training distribution, we would hope to see a similar generalization to CIFAR10 when CIFAR10 is unseen.
As visible in both plots, learning speed decreases at some point.
This is probably due to the current short-horizon bias as discussed in the previous paragraph.
Future improvements are necessary to further scale VSML to harder learning problems.

\section{Other Training Details}\label{sec:training_details}

\paragraph{LSTM implementation}
We implement the {\sapproach} using $A \cdot B$ LSTMs with forward and backward messages as described in \autoref{eq:message-forward-backward}.
Each LSTM $ab$ at layer $k$ is updated by
\begin{equation}
    z_{ab}^{(k)}, h_{ab}^{(k)} \leftarrow f_{\textrm{LSTM}}(z_{ab}^{(k)}, h_{ab}^{(k)}, \sfmsg_a^{(k)}, \sbmsg_b^{(k)}).
\end{equation}
The functions $\fmsg$ and $\bmsg$ are a linear projection to outputs of size $N' = 8$ and $N'' = 8$ respectively.
The state size is given by $N = 64$ for LA cloning and $N = 16$ for meta learning from scratch.
$A^{(1)}$ and $B^{(K)}$ are fixed according to the dataset input/output size and others are chosen freely as described in the respective experiment.
We found that averaging messages instead of summing them, $\sfmsg_b^{(k)} := \frac{1}{A^{(k-1)}} \sum_{a'} \fmsg(s_{a'b}^{(k-1)})$ and $\sbmsg_a^{(k)} := \frac{1}{B^{(k+1)}} \sum_{b'} \bmsg(s_{ab'}^{(k+1)})$, improves meta training stability.

\textbf{Source code} is available at \url{http://louiskirsch.com/code/vsml}.

\subsection{Learning algorithm cloning}

\paragraph{General training remarks}
During the forward evaluation of layers $1, \ldots, K$ we freeze the LSTM state.
During the backward pass, we only retain two state dimensions that correspond to the weight and the bias.
We also zero all other LSTM input dimensions in $\sfmsg$ and $\sbmsg$ except the ones that encode the input $x$ and error $e$.
We maintain a buffer of {\sapproach} states from which we sample a batch during LA cloning and append one of the new states to the buffer.
This ensures diversity across possible {\sapproach} states during LA cloning.

\paragraph{Batching for {\approach}}
In \autoref{sec:backprop} we optimize a {\sapproach} to implement backpropagation.
To stabilize learning at meta test time we run the RNN on multiple data points (batch size $64$) and then average their states corresponding to $w$ and $b$ as an analogue to batching in standard gradient descent.

\paragraph{Stability during meta testing}
To prevent exploding states during meta testing we also clip the LSTM state between $-4$ and $4$.

\paragraph{Bounded states in LSTMs}
In LSTMs the hidden state is bounded between $(-1, 1)$.
For learning algorithm cloning, we'd like to support weights and biases beyond this range.
This can be circumvented by choosing a constant, here $4$, by which we scale $w$ and $b$ down to store them in the context state.
This is only relevant during learning algorithm cloning.

\subsection{Meta learning from scratch}

\paragraph{Hyperparameter search strategy}
The {\principle} hyper-parameters were searched using wandb's~\citep{biewald2020} Bayesian search during development.
Parameters that lead to stable meta learning on MNIST were chosen.
The final parameters were not further tuned and doing so may lead to additional performance gains.
For the {\metarnn} we picked parameters that matched {\sapproach} as much as possible.
For our SGD and SGD with Adam baselines, we performed a grid search over the learning rate on MNIST to find the best learning rates.

\paragraph{Meta Training}
Meta training is done across 128 GPUs using ES as proposed by OpenAI~\citep{Salimans2017} for a total of 10k steps.
We use a population size of 1024, each population member is evaluated on one trajectory of 500 online examples.
We use noise with a fixed standard deviation of $0.05$.
To apply the estimated gradient, we use Adam with a learning rate of $0.025$ and betas set to $0.9$ and $0.999$.
We have run similar experiments (where GPU memory is sufficient) with distributed gradient descent on 8 GPUs which led to less stable training but qualitatively similar results with appropriate early stopping and gradient clipping.

\paragraph{{\sapproach} architecture}
Each sub-RNN has a state size of $N = 16$ and messages are sized $N' = N'' = 8$.
We only use a single layer between the input and prediction, thus $A$ equals the flattened input image dimension and $B = 10$ for the predicted logits.
The outputs are squashed between $\pm 100$ using tanh.
We run this layer two ticks per input.
States are initialized randomly from independent standard normals.

\paragraph{SGD baseline architecture and learning rate}
The deep SGD baseline uses a hidden layer of size $160$, resulting in approximately $125k$ parameters on MNIST to match the number of state dimensions of the {\sapproach}.
We use a tanh activation function to match the LSTM setup.
The tuned learning rate used for vanilla SGD is $10^{-2}$ and $10^{-3}$ for Adam.

\paragraph{{\metarnn} baseline}
We use an LSTM hidden size of $16$ and an input size of $|\mathrm{image}| + |\mathrm{error}|$ where $|\mathrm{error}|$ corresponds to the output size.
Inputs are padded to be equal size across all meta training datasets.
This results in about $100k$ to $150k$ parameters.

\paragraph{Hebbian fast weight baseline}
We compare to a Hebbian fast weight baseline as described in \citet{Miconi2018DifferentiableBackpropagation} where a single layer is adapted using learned synaptic plasticity.
A single layer is adapted using Oja's rule by feeding the prediction errors and label as additional inputs.

\paragraph{Specialization through RNN coordinates}
In addition to the recurrent inputs and inputs from the interaction term, each sub-RNN can be fed its coordinates $a, b$, position in time, or position in the layer stack.
This may allow for (1) specialization, akin to the specialization of biological neurons, and (2) for implicitly meta learning neural architectures by suppressing outputs of sub-RNNs based on positional information.
In our experiments, we have not yet observed any benefits of this approach and leave this to future work.

\paragraph{Meta learning batched LAs}
In our meta learning from scratch experiments, we discovered online learning algorithms (similar to Meta RNNs~\citep{Hochreiter2001,Wang2016,Duan2016}).
We demonstrated high sample efficiency but the final performance trails the one of batched SGD training.
In future experiments, we also want to investigate a batched variant.
Every tick we could average a subset of each state $s_{ab}$ across multiple parallel running {\approach}.
This would allow for meta learning batched LAs from scratch.

\paragraph{Optimizing final prediction error vs sum of all errors}
In our experiments we are interested in sample efficient learning, i.e., the model making good predictions as early as possible in training.
This is encouraged by minimizing the sum of all prediction errors throughout training.
If only good final performance is desired, optimizing solely final prediction error or a weighting of prediction errors is an interesting alternative to be investigated in the future.

\paragraph{Recursive replacement of weights}
Variable sharing in NNs by replacing each weight with an LSTM introduces new meta variables $V_M$.
Those variables themselves may be replaced again by LSTMs, yielding a multi-level hierarchy with arbitrary depth.
We leave the exploration of such hierarchies to future work.

\paragraph{Alternative sparse shared weight matrices}
In this paper, we have focused on a version of VSML where the sparse shared weight matrix is defined by many RNNs that pass messages.
Alternative ways of structuring variable sharing and sparsity may lead to different kinds of learning algorithms.
Investigating these alternatives or even searching the space of variable sharing and sparsity patterns are interesting directions for future research.

\paragraph{Meta Testing algorithm}
Meta testing corresponds to unrolling the {\approach}.
The learning algorithm is encoded purely in the recurrent dynamics.
See \autoref{alg:meta_testing} for pseudo-code.
\begin{algorithm}[H]
    \centering	
    \begin{algorithmic}	
        \Require Dataset $D = \{(x_i, y_i)\}$, LSTM parameters $V_M$
        \State $V_L = \{s_{ab}^{(k)}\} \leftarrow$ initialize LSTM states $\quad \forall a,b,k$
        \For{$(x, y) \in \{(x_1, y_1), \ldots, (x_T, y_T)\} \subset D$} \Comment{Inner loop over $T$ examples}
            \State $\sfmsg^{(1)}_{a1} := x_a \quad \forall a$ \Comment{Initialize from input image x}
            \For{$k \in \{1, \ldots, K\}$} \Comment{Iterating over $K$ layers}
                \State $s_{ab}^{(k)} \leftarrow f_{RNN}(s_{ab}^{(k)}, \sfmsg_a^{(k)}, \sbmsg_b^{(k)}) \quad \forall a,b$ \Comment{\autoref{eq:message-forward-backward}}
                \State $\sfmsg_b^{(k+1)} := \sum_{a'} \fmsg(s_{a'b}^{(k)}) \quad \forall b$ \Comment{Create forward message}
                \State $\sbmsg_a^{(k-1)} := \sum_{b'} \bmsg(s_{ab'}^{(k)}) \quad \forall a$ \Comment{Create backward message}
            \EndFor
            \State $\hat y_a := \sfmsg_{a1}^{(K+1)} \quad \forall a$ \Comment{Read output}
            \State $e := \nabla_{\hat y} L(\hat y, y)$ \Comment{Compute error at outputs using loss $L$}
            \State $\sbmsg_{b1}^{(K)} := e_b \quad \forall b$ \Comment{Input errors}
        \EndFor
    \end{algorithmic}	
    \caption{{\principle}: Meta Testing\label{alg:meta_testing}}
\end{algorithm}

\section{Other relationships to previous work}\label{sec:relationships}

\subsection{VSML as distributed memory}

Compared to other works with additional external memory mechanisms~\citep{sun1991neural,mozer1993connectionist,Santoro,Mishra,schlag2021learning}, VSML can also be viewed as having memory distributed across the network.
The memory writing and reading mechanism implemented in the meta variables $V_M$ is shared across the network.

\subsection{Connection to modular learning}

Our sub-LSTMs can also be framed as \emph{modules} that have some shared meta variables $V_M$ and distinct learned variables $V_L$.
Previous works in modular learning~\citep{Shazeer2017,rosenbaum2018routing,Kirsch2018modular} were motivated by learning experts with unique parameters that are conditionally selected to suit the current task or context.
In contrast, VSML has recurrent modules that share the same parameters $V_M$ to resemble a learning algorithm.
There is no explicit conditional selection of modules, although it could emerge based on activations or be facilitated via additional attention mechanisms.

\subsection{Connection to self-organization and complex systems}
In self-organizing systems, global behavior emerges from the behavior of many local systems such as cellular automata~\citep{codd2014cellular} and their recent neural variants~\citep{mordvintsev2020growing,sudhakaran2021growing}.
VSML can be seen as such a self-organizing system where many sub-RNNs induce the emergence of a global learning algorithm.

\end{document}